\pdfoutput=1
\documentclass[11pt]{article}
\usepackage{authblk}
\usepackage[top=1.in, bottom=1.in, left=1.in, right=1.in]{geometry}
\usepackage{amstext,amsmath,amsthm,amssymb,amsfonts}
\usepackage{bm, yhmath}
\usepackage{bbm, comment}

\usepackage{wrapfig}
\usepackage[utf8]{inputenc}
\usepackage{graphicx} 
\usepackage{subcaption}
\usepackage{xcolor}
\usepackage{makecell}
\usepackage{lipsum}
\usepackage{url}
\usepackage[ruled]{algorithm2e} % For algorithms
\usepackage[shortlabels]{enumitem}

\usepackage{hyperref}     
\usepackage{cleveref}

\usepackage{tabulary}
\usepackage{booktabs}
\usepackage{multirow}

\usepackage{natbib}
\setcitestyle{authoryear} %Citation-related commands

\usepackage{environ, etoolbox, fancyvrb, ifthen, kvoptions}
\usepackage[bibliography=common]{apxproof}

\newtheorem*{theorem*}{Theorem}
\newtheorem{theorem}{Theorem}[section]
\newtheorem{assumption}[theorem]{Assumption}
\newtheorem{remark}[theorem]{Remark}

\newtheorem{definition}[theorem]{Definition}

\newtheoremrep{lemma}[theorem]{Lemma}
\newtheoremrep{claim}[theorem]{Claim}
\newtheoremrep{proposition}[theorem]{Proposition}

\usepackage{amsmath,amsfonts,bm}

\def\eqref#1{equation~\ref{#1}}
\def\1{\bm{1}}

\DeclareMathAlphabet{\mathsfit}{\encodingdefault}{\sfdefault}{m}{sl}
\SetMathAlphabet{\mathsfit}{bold}{\encodingdefault}{\sfdefault}{bx}{n}

\newcommand{\E}{\mathbb{E}}

\DeclareMathOperator{\sign}{sign}

\newcommand{\nn}{\textsc{NN}}

\newcommand{\mi}{\textsc{MI}}
\newcommand{\pmi}{\textsc{PMI}}

\newcommand{\RelScore}{\text{PMIScore}\xspace}
\newcommand{\RelS}{\text{PMIS}\xspace}

\title{\RelScore: An Unsupervised Approach to Quantify Dialogue Engagement}
\author[1]{Yongkang Guo\thanks{Yongkang Guo and Zhihuan Huang contributed equally to this work.}}
\author[1]{Zhihuan Huang\protect\footnotemark[1]}
\author[1]{Yuqing Kong\thanks{Corresponding author: \texttt{yuqing.kong@pku.edu.cn}}}
\affil[1]{CFCS, School of Computer Science, Peking University, China}
\date{}

\begin{document}

\maketitle

\begin{abstract}
High dialogue engagement is a crucial indicator of an effective conversation. A reliable measure of engagement could help benchmark large language models, enhance the effectiveness of human-computer interactions, or improve personal communication skills. However, quantifying engagement is challenging, since it is subjective and lacks a ``gold standard''. This paper proposes PMIScore, an efficient unsupervised approach to quantify dialogue engagement. It uses pointwise mutual information (PMI), which is the probability of generating a response conditioning on the conversation history. Thus, PMIScore offers a clear interpretation of engagement. As directly computing PMI is intractable due to the complexity of dialogues, PMIScore learned it through a dual form of divergence. The algorithm includes generating positive and negative dialogue pairs, extracting embeddings by large language models (LLMs), and training a small neural network using a mutual information loss function. We validated PMIScore on both synthetic and real-world datasets. Our results demonstrate the effectiveness of PMIScore in PMI estimation and the reasonableness of the PMI metric itself.
\end{abstract}

\section{Introduction}
You sink into the couch, drained. “My presentation was a disaster.” Your partner barely glances over. “Yeah, traffic was awful today.”

Scenes like this happen every day—two people exchanging words, but not truly connecting. Sometimes one person answers a completely different question; other times, they respond with generic, safe phrases such as “I know how you feel” or “That’s interesting,” which keep the conversation going without real engagement.

Understanding and measuring engagement in natural language is more than a philosophical question. A reliable measure could help benchmark large language models, improve tools for business communication, track the quality of personal or professional interactions, or even analyze high-stakes scenarios like presidential debates \citep{see2019makes, yi2019towards, guo2018topic}. However, quantifying engagement is subjective, often lacks a gold standard. Human annotations are noisy, expensive, and inconsistent. Directly asking a language model to score engagement is also unstable, hard to interpret, and can feel like a black-box solution \citep{ferron2023meep}.

In this paper, we focus on two research questions:

\paragraph{RQ1:} Can we propose a metric to quantify the engagement of a dialogue?

\paragraph{RQ2:} Do we have an efficient, automatic approach to calculate the metric without human annotations?

To address the questions, this paper proposes a pointwise mutual information (PMI) approach to measure engagement between conversational context and responses. PMI has a clear interpretation: it quantifies how much more (or less) likely a response is given the conversation context compared to its baseline probability. 
\begin{equation}\label{pmi}
\begin{split}
\mathrm{PMI}(\text{response};\text{context})&=\log \frac{\Pr[\text{response},\text{context}]}{\Pr[\text{response}]\Pr[\text{context}]}\\
&=\log \frac{\Pr[\text{response}|\text{context}]}{\Pr[\text{response}]}
\end{split}
\end{equation}

In other words, if the response is highly relevant to the context, the PMI is positive; if the response is unrelated—or even counterintuitively discouraged by the context, as in the ``traffic'' example—the PMI can be negative. Generic, low-engagement responses, while valid in isolation, tend to be predictable regardless of the context, resulting in PMI scores close to zero.

This formulation naturally distinguishes true engagement from mere relevance or repetition. While a response that simply copies the context might seem ``relevant'', it adds zero new information (low PMI). In contrast, meaningful semantic reconstruction—which reflects active listening, a key component of engagement \citep{rogers1957active}—confirms understanding without parroting, thereby yielding a high PMI score. Thus, PMIScore penalizes generic or lazy responses while rewarding substantive interaction.

Computing PMI requires estimating the joint distribution of conversation histories and responses—a task that becomes intractable in high-dimensional spaces. Unlike simple, low-dimensional signal systems, conversations are vast and complex, making naïve estimation impossible.

As large language models are based on the conditional probabilities to generate the texts, one might try using their log-probabilities directly. But this has limitations: those probabilities reflect the model’s general distribution over all contexts, not conversation-specific dynamics. They are also overly sensitive to surface-level context rather than meaningful exchange. Short texts are preferable by this method, even if they have the same semantic meaning.

There is a better path. Mutual information can be expressed as the KL-divergence between the joint distribution and the product of marginals: if context and response were independent, the two distributions would coincide. Pointwise mutual information (PMI) (Equation~\ref{pmi}) corresponds to the log-ratio of these probabilities.  

The dual form of KL-divergence provides a practical way to estimate this ratio. Instead of directly computing probabilities, we compare expectations of a function under the two distributions. If the distributions are identical, these expectations match; if not, there exists a function that maximizes their discrepancy. The optimal function corresponds exactly to the likelihood ratio, i.e., the PMI. In practice, expectations can be estimated via sampling:  
\begin{enumerate}
    \item \textbf{Joint samples:} context--response pairs from the same dialogue.  
    \item \textbf{Marginal samples:} contexts and responses drawn independently from different dialogues.  
\end{enumerate}

Viewed from another perspective, this setup resembles a classification problem. Positive samples are true context--response pairs, while negative samples are mismatched pairs. Training a discriminator to separate positives from negatives yields an approximation to PMI, quantifying how much mutual information a response shares with its conversational context. The key difference from standard classification is that our goal is not merely to minimize misclassification error, but to learn a score that directly corresponds to PMI. Achieving this requires an objective function derived from the dual form of KL-divergence.  

Large language models further ease this training. Their pretrained embeddings already encode rich semantic and contextual structure. By fine-tuning an LLM---or training a smaller classifier on top of its embeddings---we can specialize the model to this discriminative task. The fine-tuned model emphasizes conversational relevance over general fluency, aligning its internal representations with the objective of estimating mutual information.  
The detailed flow is shown in \cref{fig:flow}.

\begin{figure*}
    \centering
    \includegraphics[width=1\linewidth]{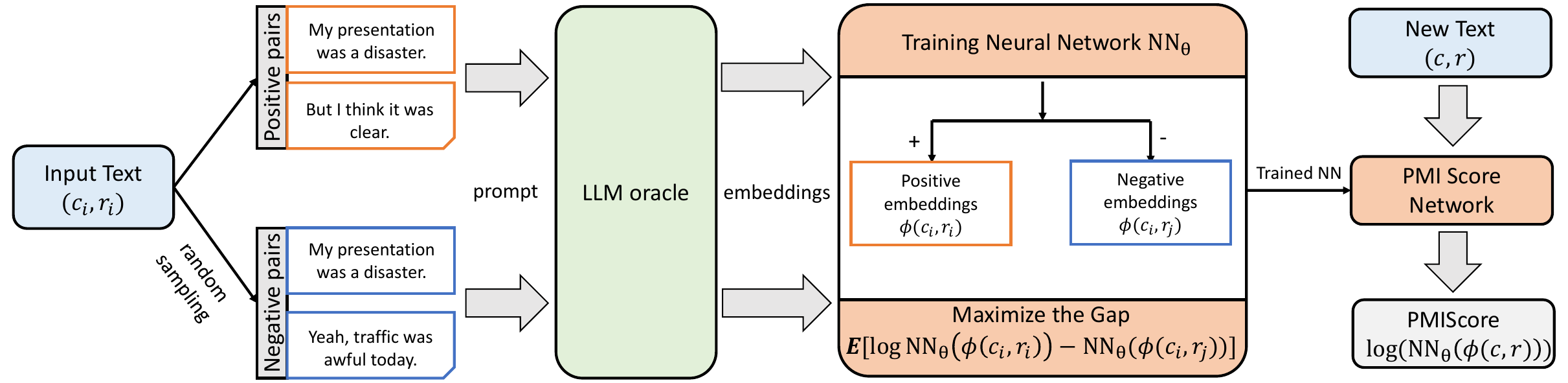}
    \caption{Flow of \RelScore. }
    \label{fig:flow}
\end{figure*}

\paragraph{Our contribution} We introduce a novel method for measuring dialogue engagement, grounded in information theory. Our approach is based on Pointwise Mutual Information (PMI), which quantifies how much new information the response contributes to the context. The method assigns a score of zero when the context and response are unrelated. For instance, if a response doesn’t offer any new or relevant insights—such as a generic, off-topic answer—it receives a score of zero, indicating no engagement. In contrast, when the response is highly relevant and aligned with the context, the score is positive, reflecting strong engagement. The score can turn negative if the response provides counterintuitive or misleading information that contradicts the context. This allows the metric to dynamically assess the quality of the response: it increases when additional context enhances the relevance of the response, but decreases when irrelevant or conflicting information detracts from it. 

We also developed a practical, unsupervised method to calculate this measure, leveraging powerful language models (LLMs) and the duality format of mutual information. We proved that with some proper assumptions, our algorithm will converge to the PMI when the sample size goes to infinity. 

Finally, we validated our approach using both synthetic datasets (created for testing) and real-world data. First, we showed that our method produced a more accurate estimation of PMI compared to previous methods. To establish a ground truth for PMI, we constructed synthetic datasets derived from various joint distributions. The results showed that our method has the lowest mean square error and the highest Pearson correlation with the true PMI. Furthermore, comparisons with state-of-the-art baselines, including the prompt-based metric MEEP \citep{ferron2023meep}, demonstrate that PMIScore achieves superior performance on both synthetic and empirical benchmarks, effectively capturing dialogue engagement. We used the multilingual dialogue collections from DSTC-11 Track~4 to construct a training set. Then we evaluated the correlation between \RelScore and the human annotated relevance scores. Our results indicated that \RelScore exhibited a high correlation with the relevance score, thereby serving as a close and reliable metric for dialogue engagement. The source code and datasets are available at \url{https://github.com/nbdhhzh/PMIScore.git}.
\section{Related work}

\paragraph{Mutual Information and Dual Form}
We use the $f$-divergence technique introduced by \cite{nguyen2009surrogate,nguyen2010estimating}. \citet{nguyen2009surrogate} extend the correspondence between surrogate loss functions and $f$-divergences in binary classification. \citet{nguyen2010estimating} then develop a method to estimate the $f$-divergence between two probability distributions via a non-asymptotic variational characterization.

Using related ideas, \citet{kong2018water} solve learning, co-training, and forecast elicitation without verification by maximizing mutual information between agents. \citet{belghazi2018mutual} introduce MINE for estimating mutual information in high-dimensional data through the dual form of KL-divergence. \citet{zheng2024proper} train Bayesian models on embedded datasets to compute mutual information. In addition, \citet{schoenebeck2020learning} propose a family of scoring rules based on $f$-divergence to incentivize truthful peer prediction.

We extend the method for estimating mutual information to the field of natural language processing using large pre-trained models. Additionally, we first propose using mutual information as a metric to measure dialogue engagement.

\paragraph{Benchmark Dialogue Systems}

Many studies are dedicated to proposing metrics for measuring the capability of dialogue systems. Previous automatic benchmark includes BLEU (\cite{papineni2002bleu}), ROUGE (\cite{lin2004rouge}), and BERTSCORE (\cite{zhang2019bertscore}). 

Nowadays, many researches explores the large language models (LLMs) to benchmark the dialogues. \cite{liu2023g} introduces G-EVAL to evaluate summarization and dialogue response generation. \cite{fu2023gptscore} use GPTSCORE to test various evaluation aspects with zero-shot instruction. \cite{ferron2023meep} propose a method to measure the engagement of dialogues by prompting the large language models. 

Previous methods typically use prompt engineering to directly obtain a score from LLMs as a metric. Our approach utilizes an unsupervised fine-tuning method that uses mutual information as the metric, which offers better interpretability.

\paragraph{Mutual Information in Dialogues}

Numerous studies investigate how mutual information can be utilized to enhance the effectiveness of machine learning models, especially in the NLP system. \cite{santra2022representation} proposes a novel Discourse Mutual Information (DMI) loss function for pretraining dialog representation models, which capture conversational structure and response prediction uncertainty.  \cite{zhang2023dialogmi} introduce a dialogue generation framework Dialog-MI which use the mutual information loss to generate more informative dialogue. \cite{ren2023c} uses the conditional pointwise mutual information (C-PMI) to  measure the turn-level interaction a dialogue systems. \cite{su2024joint} train both natural language understanding and generation systems by maximizing the mutual information between two components. \cite{xu2024benchmarking} also use a mutual information estimator to benchmark the judgment, especially for LLMs.

These researches often use mutual information as a training loss to improve the capability of a dialogue system, whereas we primarily use it as a metric to benchmark a dialogue system. Furthermore, previous methods typically used the upper bound of mutual information as an estimator, whereas we used its dual form, which allows for more accurate estimation of mutual information.
\section{Preliminaries}
In this section, we introduce some basic knowledge of mutual information that will be used in our relevance score and alogrithm. We focus on the Shannon mutual information in the rest of our paper. We also consider a more general class of mutual information called $f$-mutual information. Similarly, we can also define the $f$-divergence and its duality. The details are deferred to the appendix.

\begin{definition}[Mutual Information]\label{def:mif}
Given two random variables $X,Y$, the (Shannon) mutual information between $X$ and $Y$ is defined as
\[
    \mi(X;Y)=\sum_{x,y}\Pr(x, y)\log \left(\frac{\Pr(x,y)}{\Pr(x)\Pr(y)}\right).
\]

\end{definition}

Mutual information (MI) quantifies the statistical dependence between two random variables \(X\) and \(Y\). To analyze the association between specific outcomes (e.g., two texts), we use the \emph{pointwise mutual information (PMI)}.  

\begin{definition}[Pointwise Mutual Information]
For realizations \(X = x\) and \(Y = y\), the pointwise mutual information is defined as  
\[
    \pmi(x;y) = \log \frac{\Pr(x,y)}{\Pr(x)\Pr(y)} .
\]  

In other words, \(\pmi(x;y)\) measures how much more likely the pair \((x,y)\) is to occur together than if \(x\) and \(y\) were independent.  

Moreover, mutual information is the expectation of PMI:  
\[
    \mi(X;Y) \;=\; \E_{(x,y) \sim \Pr(X,Y)} \big[ \pmi(x;y) \big].
\]
\end{definition}

\paragraph{Properties of PMI:}
\begin{itemize}
\item \textbf{Symmetry} $\pmi(x;y)=\pmi(y;x)$.
    \item \textbf{Vanishing under independence:} $\pmi(x;y)=0$ if $X$ and $Y$ are independent.
    \item \textbf{Chain rule:} $\pmi(x,y;z) = \pmi(x;z) + \pmi(y;z \mid x)$, with conditional $\pmi(y;z \mid x) = \log \frac{\Pr(y,z \mid x)}{\Pr(y\mid x)\Pr(z\mid x)}$.
    \item \textbf{Monotonicity:} $\pmi(x,y;z) \ge \pmi(x;z)$ if $\pmi(y;z \mid x)\ge 0$.
    \item \textbf{Sign:} $\pmi(x;y)>0$ if $x$ and $y$ co-occur more than expected; $\pmi(x;y)<0$ if less.
\end{itemize}

\begin{definition}[Dual Form of Mutual Information]
The \emph{KL-divergence} between two distributions $\mathbf{p}$ and $\mathbf{q}$ is
\[
d_{\mathrm{KL}}(\mathbf{p}\,\|\,\mathbf{q}) \;=\; \E_{\mathbf{p}}\!\left[\log \frac{\mathbf{p}}{\mathbf{q}}\right].
\]
It admits the \emph{dual form}
\[
d_{\mathrm{KL}}(\mathbf{p}\,\|\,\mathbf{q}) \;=\; \sup_{D \in \mathcal{D}} \; \E_{\mathbf{p}}[\log D + 1] \;-\; \E_{\mathbf{q}}[D],
\]
where $\mathcal{D}$ denotes the set of measurable functions $D:\mathcal{X}\to \mathbb{R}_{>0}$ and $\mathcal{X}$ denotes the underlying sample space of $\mathbf{p}$ and $\mathbf{q}$. Since mutual information is the KL-divergence between the joint distribution $\Pr(X,Y)$ and the product of marginals $\Pr(X)\Pr(Y)$, we obtain the \emph{dual form of mutual information}:
\begin{equation}\label{eq:dual}
\begin{split}
\mi(X;Y) \;
=\; \sup_{D \in \mathcal{D}} \; &\big(\E_{(x,y)\sim \Pr(X,Y)}\!\big[\log D(x,y) + 1\big] \;\\
&-\; \E_{(x,y)\sim \Pr(X)\Pr(Y)}\!\big[D(x,y)\big]\big).
\end{split}
\end{equation}

Intuitively, this variational form interprets mutual information as the maximal advantage of a discriminator $D$ in distinguishing dependent pairs from independent ones. At the optimum,
\[
D^\star(x,y) \;=\; \frac{\Pr(x,y)}{\Pr(x)\Pr(y)} \;=\; e^{\pmi(x;y)},
\]
which recovers the pointwise mutual information.
\end{definition}

\paragraph{Why use the dual form?} 
The direct definition of mutual information requires access to the joint probability $\Pr(x,y)$ and the marginals $\Pr(x), \Pr(y)$. 
In high-dimensional or continuous spaces, these distributions are intractable to compute or estimate accurately.

The dual form avoids this difficulty. Instead of explicitly estimating probability densities, it reframes mutual information as a variational optimization problem. Here, we only need \emph{samples} from the joint distribution $\Pr(X,Y)$ (obtained from paired data) and from the product distribution $\Pr(X)\Pr(Y)$ (obtained by resampling or shuffling).  

This variational perspective is the foundation of mutual information estimators such as \emph{MINE} (Mutual Information Neural Estimation) and related approaches.

\section{Framework of \RelScore}
In this section, we introduce the \textbf{\RelScore} framework, which assigns a relevance score to each dialogue pair. A detailed description of the implementation is provided in the experiments section.  

Conceptually, \RelScore estimates the pointwise mutual information (PMI) between a dialogue context and its response. Suppose we have $n$ samples drawn from the joint distribution of contexts and responses, denoted as
\[
\{(c_i, r_i)\}_{i=1}^n \sim \Pr(C,R).
\]
\RelScore aims to quantify how much information the response $r_i$ carries about its corresponding context $c_i$ compared to responses drawn independently from the marginal distribution.

\begin{comment}
\begin{algorithm}
    \caption{Relevance Score Function}
    \label{alg:rel}
    \KwIn{A set of texts $\{(c_i,r_i)\}_{i=1}^n$}
    \KwOut{A score function $\RelS(c,t)$}
    \Begin{
        Calculate the embeddings $e_i=\ell(c_i,r_i)$ from a LLM oracle $\ell$\;
        optimize a neural network $\nn_\theta$ by $\frac{1}{n}\sum_i \partial f(\nn_\theta(l(c_i,r_i)))-\frac{1}{n(n-1)}\sum_{i\ne j}f^*(\nn_\theta(l(c_i,r_j)))$\;
        Calculate the score $\RelS(c,t)=f(\nn_\theta(\ell(c,t)))$\;
    return result\;
  }
\end{algorithm}
\end{comment}

\subsection{Data Preparation}

There are $n$ data points $(c_i,r_i)$ in the corpus, where $c_i\in C$ is the context of a dialogue and $r_i\in R$ is the response text given $c_i$. For example, the context $c_i$ usually contains the previous rounds of a dialogue and a question from another person, while $r_i$ is the answer to that question. 

The original dataset $D^+=\{(c_i,r_i)\}_i$ is regarded as the \textit{postive pairs}. and we construct the \textit{negative pairs} $D^-=\{(c_i,r_j)\}_{i,j}$ by randomly sampling some response texts $r_j$ from the dataset for each context $c_i$. As the response is meaningful but not to the context, we regard them as negatively relevance.

\subsection{LLMs-Based Embeddings}

Large Language Models (LLMs) are highly effective and powerful generative models. They are typically composed of two components: encoder blocks and decoder blocks. In traditional generation tasks, an LLM first receives an input. This input is processed by the encoder network to produce a vector (also called an embedding), which serves as a compressed representation of the input's information. Subsequently, this vector is passed through the decoder component, which then generates the response word by word in an autoregressive manner.

Given any pair of texts $d\in C\times R$, we use the encoder blocks of LLMs to extract the semantic meaning of the text. An LLM oracle is defined by $\phi: P\times C\times R\to \mathbb{R}^m$. It takes a task description (also called Prompt) $p$ and a pair of data $d$ as input and outputs a $m$-dimensional embedding $e$ \cite{jiang2023scaling}.

\subsection{Training the Small Neural Network}

Building on the \RelScore framework introduced above, we train a small neural network $\nn_\theta: \mathbb{R}^m \to \mathbb{R}$ to compute the final relevance score for a given dialogue pair, where $\theta$ denotes the network parameters. We apply a softcap operator as a form of regularization to enforce Lipschitz continuity and improve robustness against outliers (e.g., false negatives in sampling).
The network is implemented as a multi-layer perceptron (MLP), making it significantly smaller and more efficient than full-scale LLMs.

Let $D^+$ denote the set of positive (context, response) pairs and $D^-$ denote negative pairs. To align the network outputs with pointwise mutual information, we define a mutual information dual form–based loss function:
\begin{align*}
    L_{\mi}(\theta) = - \Big( \E_{d \in D^+} \log \nn_\theta(\phi(p^*, d)) - \E_{d \in D^-} \nn_\theta(\phi(p^*, d)) \Big),
\end{align*}
where $p^*$ is a predetermined prompt and $\phi(\cdot)$ represents the input feature embedding. Details of the prompt and feature design are provided in the appendix. 

Unlike traditional pairwise learning methods, this loss is asymmetric with respect to positive and negative pairs. This formulation directly corresponds to the dual-form of mutual information in Equation~(\ref{eq:dual}), specifically the NWJ bound proposed by \citet{nguyen2010estimating}, with the only difference being that the constant term $+1$ from the original dual form is omitted. Since this constant does not affect the optimization, it can safely be dropped. As in the dual-form interpretation, this loss encourages the network to assign high scores to true context–response pairs while lowering the scores of mismatched pairs, effectively approximating the pointwise mutual information.

\subsection{New Pairs Evaluation}

After we trained the optimal parameters $\theta^*=\arg\min L_{\mi}(\theta)$, we will use it to evaluate any pair of dialogues. Given the dialogue pair $c,r$, its \RelScore is given by:

$$\RelS(c,r)=\log \nn_{\theta^*}\left(\phi\left(p^*,(c,r)\right)\right).$$

\begin{comment}
    
\paragraph{Discussion}

In practice, the information conveyed by sentences always depends on the context. For instance, the same words can carry different amounts of information during a daily conversation compared to a technical presentation. Different people also have different ways of expressing the same meaning. Therefore, we use $\mathcal{F}$ to represent all the current contextual information, and we measure the conditional mutual information $MI^f(X;Y|\mathcal{F})$.

In our experiments, we will assume that the same dataset implies the similar context. Therefore, we measure the conditional mutual information $MI^f(X;Y|\mathcal{D})$ based on the dataset $\mathcal{D}$. For simplicity, we'll omit the condition $\mathcal{D}$.

\end{comment}
\section{Theoretical Results}

In this section, we prove that our algorithm will learn the PMI of each pair. Thus it is a proper metric for the relevance score. The main proof idea is borrowed from \citet{belghazi2018mutual}. Due to limited space, we defer the proofs to \Cref{apx:proof}. 

We first make two assumptions.

\begin{assumption}[Semantic Embedding Assumption]
    The LLM embedding $\phi(c,r)$ can extract the semantic meaning. That is, for any $c_1, c_2$ and $r_1, r_2$ with the same meaning, $\phi(c_1,r_1)=\phi(c_2,r_2)$.
\end{assumption}

\begin{assumption}[I.I.D. Samples Assumption]
    The data points $\{(c_i, r_i)\}_i$ are independent samples drawn from a joint distribution $\Pr(C=c,R=r)$.
\end{assumption}

Notice that given the independent samples assumtion, there are two underlying random variables $C,R$ of the context $c$ and response $r$ seperately. Now we state our main theorem which shows minimize our loss function $L_{\mi}$ is equivalent to estimate the mutual information of $C,R$. In addition, the minimizer corresponds to the PMI.

% \begin{theorem}
%     Let $NN_{\theta^*}(l(c,r))=\pmi(c,r)=\frac{\Pr(X=c,Y=r)}{\Pr(X=c)\Pr(Y=r)}$, then $\theta^*$ is the minimizer of 

%     $$\min_{\theta}L_{\mi}(\theta),$$ and the expected minimum is the opposite of mutual information $\E_{X,Y}L_{\mi}(\theta^*)=-\mi(X;Y)$.
% \end{theorem}

\begin{theorem}
\label{thm:convergence}
Under the assumptions described above, as the number of samples $n \to \infty$, the \RelScore converges to the pointwise mutual information:
\[
\RelScore(c,r) \;\longrightarrow\; \pmi(c,r).
\]
\end{theorem}

\begin{remark}
The \RelScore\ is symmetric with respect to context and response. A context-independent, ``one-size-fits-all'' response, which carries no information about the specific context, has 
\(\Pr(Y=r \mid X=c) = \Pr(Y=r)\) and thus \(\RelScore(c,r) = 0\). 
In contrast, responses that are highly context-specific yield positive scores.

The \emph{chain rule} of PMI implies that adding additional contextual information can further increase or decrease the pointwise mutual information. For example, if we extend the response from \(r\) to \((r, z)\), then
\[
\pmi(c;(r,z)) = \pmi(c; r) + \pmi(c;z\mid r),
\]

The score increases when the additional context \(z\) provides positive information about the context, i.e., \(\pmi(c;z\mid r) > 0\), and decreases when \(z\) provides misleading or negatively correlated information, i.e., \(\pmi(c;z\mid r) < 0\). This highlights that \(\RelScore\) dynamically reflects the incremental contribution of new texts to the relevance of a context.

While Theorem \ref{thm:convergence} relies on ideal semantic extraction (Assumption 5.1), the method holds under practical conditions where the embedding space exhibits semantic continuity. Specifically, if similar meanings map to proximal vectors, the algorithm converges to the true PMI with an error bound proportional to the embedding distortion. This theoretical robustness is supported by our extensive experiments (see Appendix), where PMIScore maintains consistent performance across diverse LLM architectures and embedding dimensions.
\end{remark}

\section{Synthetic Study}
\begin{figure*}[!ht]
    \centering
    \includegraphics[width=0.95\linewidth]{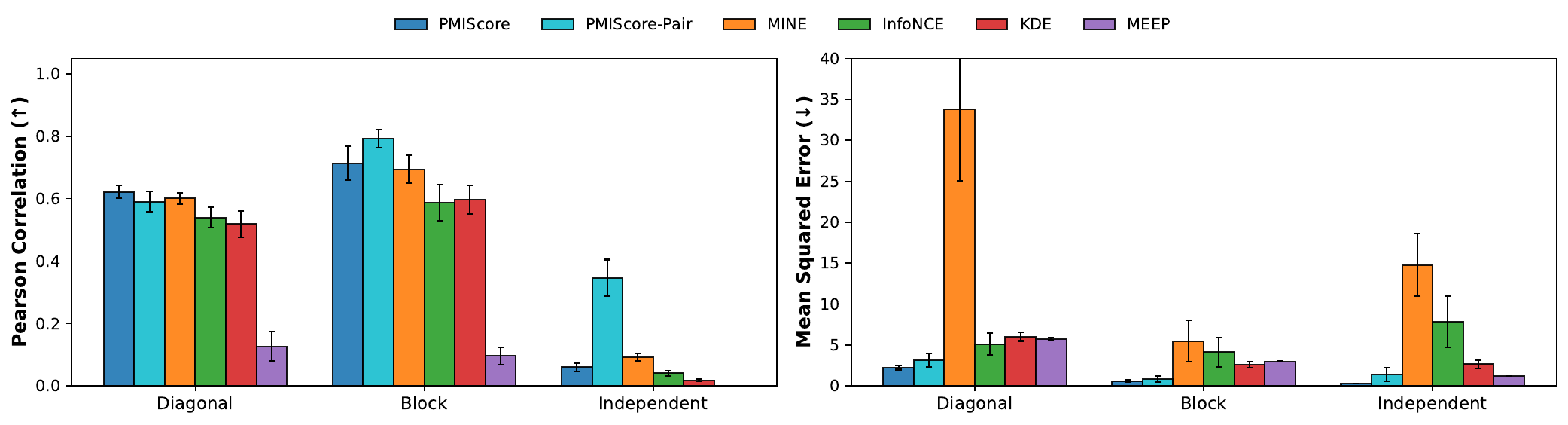}
    \caption{\textbf{\RelScore vs. baselines on synthetic distributions.}
    The figure shows the Spearman correlation (left) and mean squared error (right) of different methods (\RelScore, MINE, InfoNCE, and KDE) on three synthetic dependency structures (\emph{Diagonal}, \emph{Block}, and \emph{Independent}). Higher Spearman values indicate better alignment with the true PMI ranking and lower MSE means better approximation of the accurate PMI. \RelScore achieves consistently highest rank correlation and lowest MSE across almost all settings.}
    \label{fig:spearman_avg}
\end{figure*}
We run a controlled synthetic study to test whether \textsc{\RelScore} can recover the true pointwise mutual information (PMI) between a context and a response.
Unlike the real dialogue datasets used later, this setting provides analytic ground truth for PMI.
% The full implementation, including data generation and evaluation, is part of the same notebook as the real-data experiments.
\subsection{Experimental Setup}
\paragraph{Synthetic Distributions}
We build three kinds of joint distributions $p(C,R)$ with different dependency patterns between context $C$ and response $R$:
\begin{enumerate}
    \item \textbf{Diagonal:} one-to-one aligned pairs ($C_i \leftrightarrow R_i$);
    \item \textbf{Block:} pairs correlated within several topic blocks but independent across blocks;
    \item \textbf{Independent:} $C$ and $R$ sampled independently from their marginals.
\end{enumerate}
Each distribution yields an analytic PMI target.
For every distribution we sample 5{,}000 pairs and split them into 60\% train, 20\% validation, and 20\% test.

\paragraph{Negative Sampling}
The training process uses both positive and negative pairs.
\textbf{Positives} are the true $(C,R)$ pairs sampled from $p(C,R)$.
\textbf{Negatives} are used only during training and are formed by randomly mismatching $C$ and $R$ within the training positives.
We pre-generate a pool of 15 negatives per positive sample and iterate through them in 5 training rounds (using 3 distinct negatives per round). This mirrors real scenarios where we must derive negatives from observed data without access to $p(C)$ or $p(R)$.
The validation and test sets contain only positives.

\paragraph{Embeddings and Scoring Network}
Each context–response pair $(c,r)$ is represented using sentence embeddings from several pretrained LLMs\footnote{Including Llama-3.2-3B, Phi-4-mini, Qwen3-0.6B/1.7B/4B/8B.}. 
A lightweight multilayer perceptron maps these embeddings to a scalar score approximating the true pointwise mutual information (PMI). More details about the training setting are shown in \Cref{app:synthetic}.

\subsection{Compared Methods}
We compare \textsc{\RelScore} with three mutual information estimators in prior work: \textbf{MINE}, \textbf{InfoNCE}, and \textbf{KDE}. 
We also include \textbf{MEEP}~\cite{ferron2023meep}, a state-of-the-art prompt-based method for dialogue engagement evaluation, as an additional baseline.

\textbf{MINE}~\cite{ishmael2018mine} is a variational mutual information estimator that optimizes the Donsker–Varadhan lower bound on the Kullback–Leibler divergence between the joint and product-of-marginals distributions. 
It relies on training a neural critic to maximize this bound, yielding flexible yet potentially unstable estimates when data are high-dimensional.

\textbf{InfoNCE}~\cite{oord2018representation} formulates mutual information estimation as a contrastive prediction task: given a context $c$, the correct response $r^+$ must be distinguished from multiple negative samples $\{r^-\}$. 
The objective encourages higher similarity for true pairs and penalizes confusion with negatives via a temperature-scaled softmax. 
This formulation is efficient and widely used in self-supervised learning but is sensitive to negative sampling quality and temperature tuning.

\textbf{KDE}~\cite{parzen1962estimation} is a classical non-parametric baseline that estimates $p(c,r)$ and $p(c)p(r)$ through kernel density estimation, and then computes $\log p(c,r) - \log p(c)p(r)$. 
While conceptually simple, KDE struggles in high-dimensional embedding spaces, where density estimation becomes unreliable.

\textbf{MEEP}~\cite{ferron2023meep} utilizes large language models (LLMs) to directly score dialogue engagement via prompting. It leverages the semantic understanding capabilities of LLMs to assess the quality and relevance of responses.

Together, these baselines span neural, contrastive, statistical, and prompt-based paradigms, enabling a controlled comparison with \textsc{\RelScore}.

\paragraph{Evaluation Metrics}
Performance is evaluated primarily by \textbf{Mean Squared Error (MSE)} between predicted and ground-truth PMI, complemented by \textbf{Spearman Correlations} to assess the rank consistency.

\subsection{Results and Analysis}
As shown in \Cref{fig:spearman_avg}, \textsc{\RelScore} consistently achieves lower MSE and higher rank correlation ($\rho$) than the baselines across all dependency structures (\textit{Diagonal}, \textit{Block}, and \textit{Independent}).
This indicates that \textsc{\RelScore} not only recovers the correct ordering of mutual information strength but also provides more accurate quantitative estimates. 
While InfoNCE performs reasonably well on rank correlation, it struggles to recover the true scale of PMI (higher MSE). This is because the InfoNCE objective approximates a lower bound on mutual information ($\log K$) rather than the pointwise value itself, and its logits are unnormalized log-density ratios up to a constant shift. In contrast, \RelScore, via the dual-form objective, directly estimates the density ratio without such constraints, capturing the absolute magnitude of information.
MEEP, while leveraging the power of LLMs, shows lower correlation and higher MSE compared to \RelScore, suggesting that direct prompting may not fully capture the nuanced statistical dependencies measured by PMI.
Figure~\ref{fig:synthetic_scatter} further illustrates the estimation behavior on the \textit{Block} dataset: predictions from \textsc{\RelScore} align closely with the identity line, while MINE and InfoNCE exhibit systematic bias or slope distortion, and KDE shows wide dispersion due to high-dimensional estimation noise.

Overall, the synthetic results validate that \textsc{\RelScore} aligns with the theoretical definition of PMI while remaining robust to variations in dependency structure and embedding model. 

\begin{figure*}[!ht]
    \centering
    \includegraphics[width=0.95\linewidth]{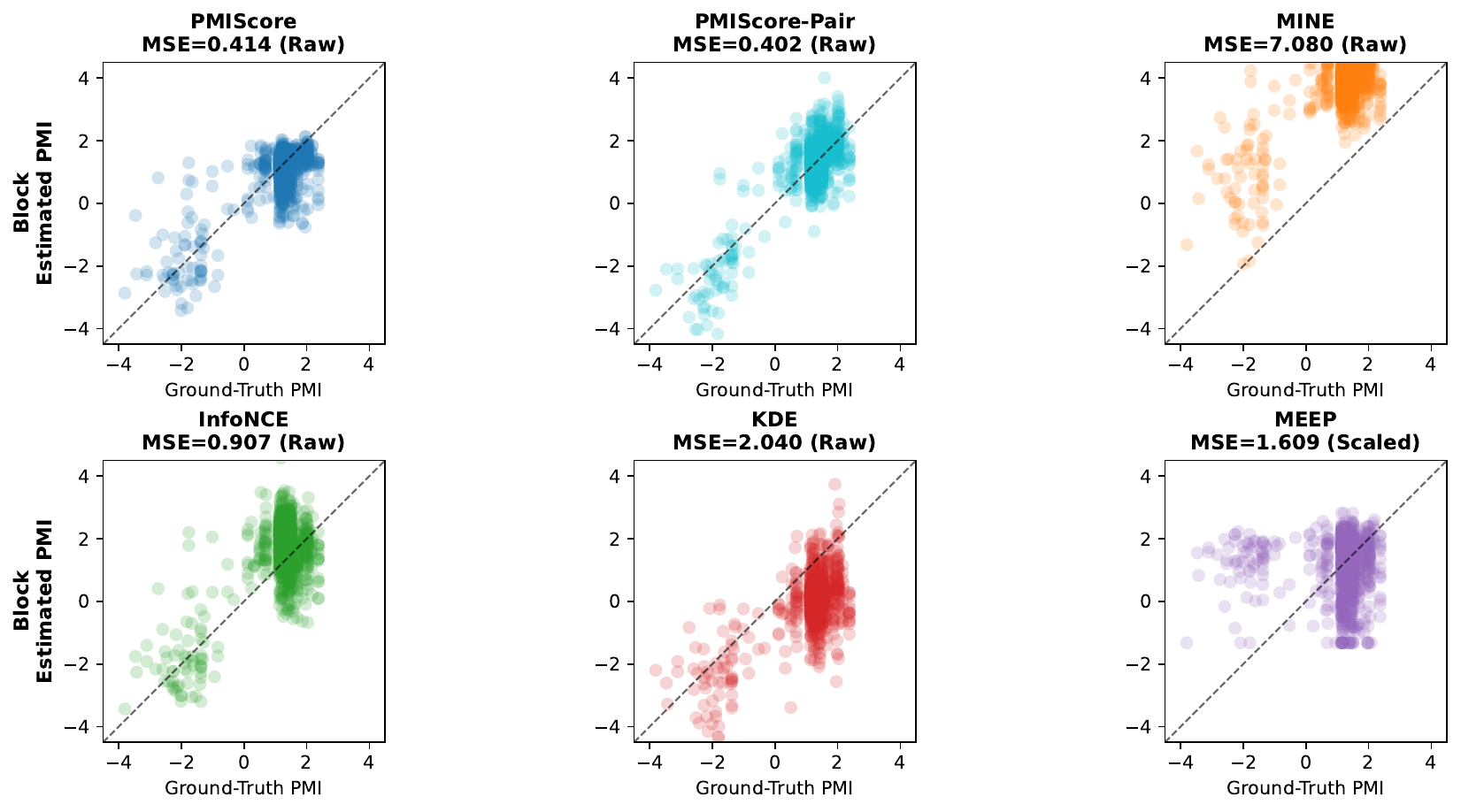}
    \caption{\textbf{Consistency of PMI estimation for different methods.}
    Each panel compares the estimated versus ground-truth pointwise mutual information (PMI) for four estimation methods—\RelScore, MINE, InfoNCE, and KDE—on the \emph{Block} synthetic dataset with the \texttt{Qwen3-4B} embeddings. 
    To illustrate the magnitude of estimation error, 1{,}000 representative samples are plotted in each figure. The x-axis is the ground-truth of PMI and the y-axis is the estimated value. The dashed line means that the estimated PMI perfectly matchs the ground truth. While \RelScore aligns closely with the identity line, indicating minimal estimation bias, alternative estimators display noticeable deviations or slope distortions. Similar patterns are also observed across other embedding models and synthetic datasets.}

    \label{fig:synthetic_scatter}
\end{figure*}

\section{Empirical Study}
\label{sec:empirical}
\begin{figure*}[!ht]
    \centering
    \includegraphics[width=0.95\linewidth]{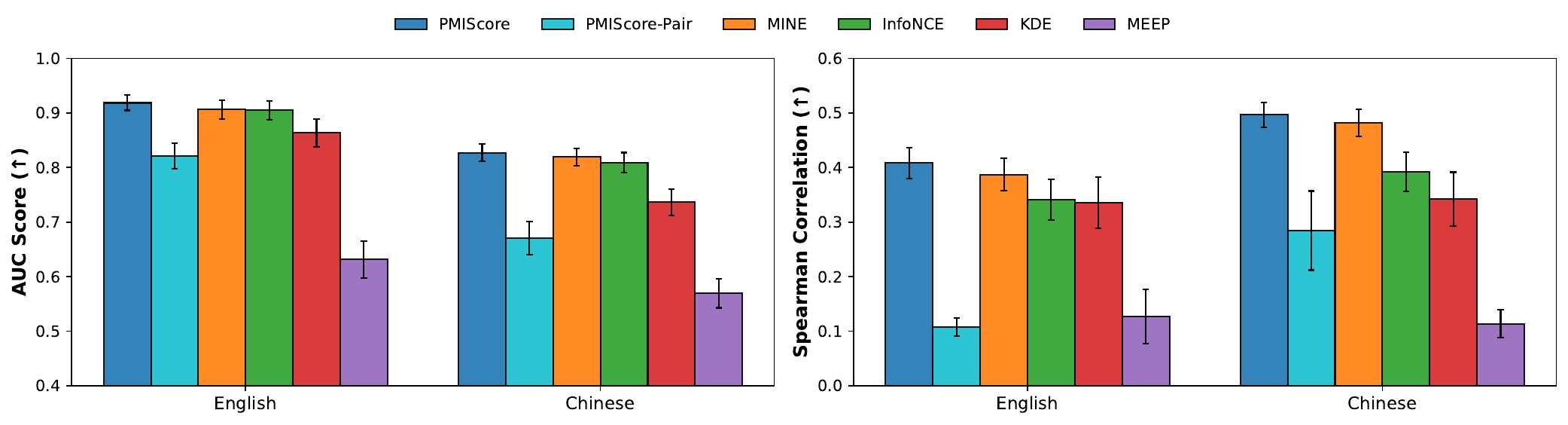}
    \caption{\textbf{\RelScore vs.\ baselines on empirical dialogue datasets.}
    Bars show mean across different llms for English (\texttt{dstc\_en}) and Chinese (\texttt{dstc\_zh}).
    \emph{Left:} ROC-AUC on \emph{val/test} for the response-ranking task. 
    \emph{Right:} Spearman correlation with human relevance on the dev sets.
    Higher ROC-AUC indicates that the score can better distinguish the positive and negative pairs. \RelScore closely tracks \textbf{InfoNCE} on AUC and yields the best human correlation on English while remaining competitive on Chinese; both substantially outperform \textbf{KDE}.}
    \label{fig:empirical_methods_by_language_avg}
\end{figure*}

This section evaluates whether \textsc{\RelScore} learned in real multi-turn dialogues can (i) rank the true response above hard negatives and (ii) align with human-perceived relevance.
Unlike the synthetic setting---where ground-truth PMI is analytically available---here we rely on task performance and human judgments as proxies.
The pipeline mirrors the synthetic study: we embed $(c,r)$ pairs with LLMs, train a small scorer with the dual-form MI objective (\RelScore), and compare against \textbf{MINE}, \textbf{InfoNCE}, \textbf{KDE}, and \textbf{MEEP}.

\subsection{Experimental Setup}

\paragraph{Datasets.}
We evaluate on the multilingual dialogue collections from DSTC-11 Track~4 (\texttt{dstc\_en}, \texttt{dstc\_zh}; \cite{rodriguezcantelar2023dstc11t4}).
Representative English sources include \texttt{DailyDialog}~\cite{li2017dailydialog}, \texttt{PersonaChat}/ConvAI2~\cite{zhang2018personalizing}, \texttt{EmpatheticDialogues}~\cite{rashkin2019empathetic}, and \texttt{TopicalChat}~\cite{gopalakrishnan2019topical}. Representative Chinese sources include \texttt{LCCC}~\cite{wang2020lccc} and \texttt{KdConv}~\cite{zhou2020kdconv}.
Each example uses a multi-turn context $c$ (dialogue history) paired with its response $r^+$.

\paragraph{Data split.}
For each language, we randomly sample 3{,}000/1{,}000 /1{,}000 context--response pairs for \emph{train}/\emph{validation}/\emph{test} from the DSTC-11 Track~4 training pool.
Model selection is given by \emph{validation Area Under the Curve (AUC)}; all retrieval results are reported on the held-out \emph{test} split.

\textit{Negative construction.}
For each positive $(c, r^+)$ we generate 3 negative samples. Each negative sample has a 10\% probability of being drawn from the same dialogue (in-dialogue negative) and a 90\% probability of being a random response from the corpus.
This mixture yields both hard (in-dialogue) and diverse (random) distractors while avoiding label leakage.

\emph{Human annotations.}
For correlation analysis we use the human-annotated dev packs included in Track~4.
We take the averaged relevance field \texttt{annot\_relevant\_mean} as the human score and evaluate per llm on $n{=}375$ English and $n{=}1{,}589$ Chinese pairs, respectively.

\paragraph{llms and embeddings.}
We probe scaling with compact encoders: Qwen3 (0.6B/1.7B/4B/8B), Llama-3.2-3B-Instruct, and Phi-4-mini.
For each $(c,r)$, sentence embeddings are fed to a lightweight MLP scorer.
Architecture and preprocessing mirror the synthetic study for an apples-to-apples comparison.

\paragraph{Training protocol.}
The only difference from the synthetic setup is the absence of analytic PMI.
All other ingredients are identical: the negative recipe, batching, optimizer, and early stopping by validation AUC. We train for 5 rounds, resampling negatives each round.

\subsection{Metrics}

\paragraph{(1) ROC-AUC (response ranking).}
Given a context $c$, we score the gold response $r^+$ and its negatives $\{r^-\}$.
ROC-AUC is the probability that a randomly chosen positive is ranked above a randomly chosen negative; it is computed from the score distributions without choosing a threshold.
We report overall AUC and show averages across llms in Fig.~\ref{fig:empirical_methods_by_language_avg}.

\paragraph{(2) Spearman correlation (human alignment).}
To evaluate \RelScore's alignment to human labels, we compute Spearman’s $\rho$ between model scores and human annotated relevance, and we show the average across llms.

\subsection{Results}

\paragraph{Aggregated AUC}
As summarized in Fig.~\ref{fig:empirical_methods_by_language_avg}, \textbf{\RelScore} outperforms all baselines on AUC, including InfoNCE.
While InfoNCE is designed to optimize a contrastive ranking objective, \RelScore achieves superior ranking performance by learning a calibrated pointwise mutual information score. This suggests that directly estimating the density ratio (PMI) provides a more robust signal for distinguishing true responses from negatives than the standard contrastive loss, especially in the resampled negative training setting.
\textbf{InfoNCE} follows closely, which is expected given its alignment with ranking metrics, while \textbf{KDE} significantly underperforms due to the high dimensionality of the embeddings. \textbf{MEEP}, relying on direct prompting, shows weaker discrimination capability compared to the training-based methods.

\paragraph{Human alignment}
Fig.~\ref{fig:empirical_methods_by_language_avg} (right) shows that \textbf{\RelScore} attains the highest correlation with human judgments on both English and Chinese dev sets.
Unlike purely contrastive objectives like InfoNCE, which are shift-invariant (i.e., $f(x,y)$ and $f(x,y)+C$ yield the same softmax probabilities), \textbf{\RelScore} is grounded in the dual form of KL-divergence, driving the score towards the true log-likelihood ratio. This scale-sensitive calibration aligns better with human perception of engagement, where the \emph{degree} of relevance matters, not just the relative ordering.
Furthermore, \RelScore surpasses MINE and MEEP, demonstrating the superiority of the dual-form MI objective over variational bounds and zero-shot prompting for engagement evaluation.

\paragraph{Takeaways}
\textbf{(i) Ranking.} \RelScore achieves the best retrieval performance (AUC), demonstrating its robustness in distinguishing valid responses from hard negatives.
\textbf{(ii) Human alignment.} \RelScore consistently leads in alignment with human annotations across languages, validating PMI as a superior proxy for engagement compared to uncalibrated contrastive scores or direct prompting.
\textbf{(iii) Interpretability.} \RelScore estimates PMI itself; its scores carry clear semantics that transfer across backbones and languages, providing a principled alternative to contrastive-only training.

\section{Conclusion}

In this paper, we propose the \RelScore to measure the engagement of a dialogue. It generates positive and negative samples from the dialogue corpus, extracts the semantic embeddings by LLMs and trains a neural network by mutual information based loss functions. We justify its theoretical interpretation and validate its performance on both synthetic and real-world datasets. 

There are several future directions to improve our work. First, we will explore the application of PMIScore to enhance dialogue generation agents. This can be achieved by integrating it as a reward function within a reinforcement learning framework or by employing it as a fine-tuning objective for large language models (LLMs).

Second, as real-world conversations often require domain-specific adaptations in real time, such as interactions with strangers or subject matter experts, we can extend our framework to an online learning setting. This will enable the system to quickly adapt to new conversational styles and dynamics.

Finally, since our theoretical results rely on the semantic assumptions of LLMs, a critical area for future research is to investigate methods for obtaining more robust embeddings. This is essential for mitigating potential errors in PMIScore and ensuring the reliability of our model. Additionally, we note that the \RelScore framework optimizes a dual objective where gradients are not naturally bounded, unlike InfoNCE which is constrained by $\log K$. In our experiments, we addressed this by applying a softcap operator (e.g., $20\tanh(x/20)$) to the score output, serving as a regularization term to enforce Lipschitz continuity and improve robustness against outliers (e.g., false negatives in sampling). Future work could explore more theoretically grounded ways to bound these gradients or adaptive methods to set the softcap parameter.

\newpage
\bibliographystyle{apalike}
\bibliography{ref}

@article{rogers1957active,
  title={Active listening},
  author={Rogers, Carl R and Farson, Richard E},
  journal={Industrial Relations Center of the University of Chicago},
  year={1957}
}

@article{xu2024benchmarking,
  title={Benchmarking LLMs' Judgments with No Gold Standard},
  author={Xu, Shengwei and Lu, Yuxuan and Schoenebeck, Grant and Kong, Yuqing},
  journal={arXiv preprint arXiv:2411.07127},
  year={2024}
}

@article{schoenebeck2020learning,
  title={Learning and strongly truthful multi-task peer prediction: A variational approach},
  author={Schoenebeck, Grant and Yu, Fang-Yi},
  journal={arXiv preprint arXiv:2009.14730},
  year={2020}
}

@inproceedings{santra2022representation,
  title={Representation Learning for Conversational Data using Discourse Mutual Information Maximization},
  author={Santra, Bishal and Roychowdhury, Sumegh and Mandal, Aishik and Gurram, Vasu and Naik, Atharva and Gupta, Manish and Goyal, Pawan},
  booktitle={Proceedings of the 2022 Conference of the North American Chapter of the Association for Computational Linguistics: Human Language Technologies},
  pages={1718--1734},
  year={2022}
}

@article{yi2019towards,
  title={Towards coherent and engaging spoken dialog response generation using automatic conversation evaluators},
  author={Yi, Sanghyun and Goel, Rahul and Khatri, Chandra and Cervone, Alessandra and Chung, Tagyoung and Hedayatnia, Behnam and Venkatesh, Anu and Gabriel, Raefer and Hakkani-Tur, Dilek},
  journal={arXiv preprint arXiv:1904.13015},
  year={2019}
}

@article{guo2018topic,
  title={Topic-based evaluation for conversational bots},
  author={Guo, Fenfei and Metallinou, Angeliki and Khatri, Chandra and Raju, Anirudh and Venkatesh, Anu and Ram, Ashwin},
  journal={arXiv preprint arXiv:1801.03622},
  year={2018}
}

@article{see2019makes,
  title={What makes a good conversation? how controllable attributes affect human judgments},
  author={See, Abigail and Roller, Stephen and Kiela, Douwe and Weston, Jason},
  journal={arXiv preprint arXiv:1902.08654},
  year={2019}
}

@article{oord2018representation,
  title={Representation learning with contrastive predictive coding},
  author={Oord, Aaron van den and Li, Yazhe and Vinyals, Oriol},
  journal={arXiv preprint arXiv:1807.03748},
  year={2018}
}

@inproceedings{belghazi2018mutual,
  title={Mutual information neural estimation},
  author={Belghazi, Mohamed Ishmael and Baratin, Aristide and Rajeshwar, Sai and Ozair, Sherjil and Bengio, Yoshua and Courville, Aaron and Hjelm, Devon},
  booktitle={International conference on machine learning},
  pages={531--540},
  year={2018},
  organization={PMLR}
}

@article{nguyen2009surrogate,
  title={On surrogate loss functions and f-divergences},
  author={Nguyen, XuanLong and Wainwright, Martin J and Jordan, Michael I},
  year={2009}
}

@article{nguyen2010estimating,
  title={Estimating divergence functionals and the likelihood ratio by convex risk minimization},
  author={Nguyen, XuanLong and Wainwright, Martin J and Jordan, Michael I},
  journal={IEEE Transactions on Information Theory},
  volume={56},
  number={11},
  pages={5847--5861},
  year={2010},
  publisher={IEEE}
}

@article{rockafellar2015convex,
  title={Convex analysis:(pms-28)},
  author={Rockafellar, Ralph Tyrell},
  year={2015},
  publisher={Princeton university press}
}

@article{kong2019information,
  title={An information theoretic framework for designing information elicitation mechanisms that reward truth-telling},
  author={Kong, Yuqing and Schoenebeck, Grant},
  journal={ACM Transactions on Economics and Computation (TEAC)},
  volume={7},
  number={1},
  pages={1--33},
  year={2019},
  publisher={ACM New York, NY, USA}
}

@inproceedings{kong2018water,
  title={Water from two rocks: Maximizing the mutual information},
  author={Kong, Yuqing and Schoenebeck, Grant},
  booktitle={Proceedings of the 2018 ACM Conference on Economics and Computation},
  pages={177--194},
  year={2018}
}

@inproceedings{zhang2023dialogmi,
  title={DialogMI: A Dialogue Model Based on Enhancing Dialogue Mutual Information},
  author={Zhang, Yibo and Gong, Ping and Wang, Zelin and Li, Zhe and Yang, Xuanyuan},
  booktitle={ICASSP 2023-2023 IEEE International Conference on Acoustics, Speech and Signal Processing (ICASSP)},
  pages={1--5},
  year={2023},
  organization={IEEE}
}

@inproceedings{ren2023c,
  title={C-PMI: Conditional Pointwise Mutual Information for Turn-level Dialogue Evaluation},
  author={Ren, Liliang and Sidhu, Mankeerat and Zeng, Qi and Reddy, Revanth Gangi and Ji, Heng and Zhai, ChengXiang},
  booktitle={Proceedings of the Third DialDoc Workshop on Document-grounded Dialogue and Conversational Question Answering},
  pages={80--85},
  year={2023}
}

@article{su2024joint,
  title={Joint dual learning with mutual information maximization for natural language understanding and generation in dialogues},
  author={Su, Shang-Yu and Chung, Yung-Sung and Chen, Yun-Nung},
  journal={IEEE/ACM Transactions on Audio, Speech, and Language Processing},
  volume={32},
  pages={2445--2452},
  year={2024},
  publisher={IEEE}
}

@inproceedings{ferron2023meep,
  title={MEEP: Is this engaging? prompting large language models for dialogue evaluation in multilingual settings},
  author={Ferron, Amila and Shore, Amber and Mitra, Ekata and Agrawal, Ameeta},
  booktitle={Findings of the Association for Computational Linguistics: EMNLP 2023},
  pages={2078--2100},
  year={2023}
}

@inproceedings{papineni2002bleu,
  title={Bleu: a method for automatic evaluation of machine translation},
  author={Papineni, Kishore and Roukos, Salim and Ward, Todd and Zhu, Wei-Jing},
  booktitle={Proceedings of the 40th annual meeting of the Association for Computational Linguistics},
  pages={311--318},
  year={2002}
}

@inproceedings{lin2004rouge,
  title={Rouge: A package for automatic evaluation of summaries},
  author={Lin, Chin-Yew},
  booktitle={Text summarization branches out},
  pages={74--81},
  year={2004}
}

@article{zhang2019bertscore,
  title={Bertscore: Evaluating text generation with bert},
  author={Zhang, Tianyi and Kishore, Varsha and Wu, Felix and Weinberger, Kilian Q and Artzi, Yoav},
  journal={arXiv preprint arXiv:1904.09675},
  year={2019}
}

@article{liu2023g,
  title={G-eval: NLG evaluation using gpt-4 with better human alignment},
  author={Liu, Yang and Iter, Dan and Xu, Yichong and Wang, Shuohang and Xu, Ruochen and Zhu, Chenguang},
  journal={arXiv preprint arXiv:2303.16634},
  year={2023}
}

@article{fu2023gptscore,
  title={Gptscore: Evaluate as you desire},
  author={Fu, Jinlan and Ng, See-Kiong and Jiang, Zhengbao and Liu, Pengfei},
  journal={arXiv preprint arXiv:2302.04166},
  year={2023}
}

@article{Ali1966general,
  author  = {Ali, S. M. and Silvey, S. D.},
  title   = {A general class of coefficients of divergence of one distribution from another},
  journal = {Journal of the Royal Statistical Society: Series B (Methodological)},
  year    = {1966},
  volume  = {28},
  number  = {1},
  pages   = {131--142},
}

@article{ishmael2018mine,
  title={MINE: mutual information neural estimation},
  author={Ishmael Belghazi, Mohamed and Baratin, Aristide and Rajeswar, Sai and Ozair, Sherjil and Bengio, Yoshua and Courville, Aaron and Devon Hjelm, R},
  journal={arXiv e-prints},
  pages={arXiv--1801},
  year={2018}
}

@article{parzen1962estimation,
  title={On estimation of a probability density function and mode},
  author={Parzen, Emanuel},
  journal={The annals of mathematical statistics},
  volume={33},
  number={3},
  pages={1065--1076},
  year={1962},
  publisher={JSTOR}
}

@article{jiang2023scaling,
  title={Scaling sentence embeddings with large language models},
  author={Jiang, Ting and Huang, Shaohan and Luan, Zhongzhi and Wang, Deqing and Zhuang, Fuzhen},
  journal={arXiv preprint arXiv:2307.16645},
  year={2023}
}

@article{zheng2024proper,
  title={Proper Dataset Valuation by Pointwise Mutual Information},
  author={Zheng, Shuran and Qi, Xuan and Chen, Rui Ray and Kwon, Yongchan and Zou, James},
  journal={arXiv preprint arXiv:2405.18253},
  year={2024}
}

@inproceedings{rodriguezcantelar2023dstc11t4,
    author    = "Mario Rodríguez-Cantelar and Chen Zhang and Chengguang Tang and Ke Shi and Sarik Ghazarian and João Sedoc and Luis Fernando D'Haro and Alexander Rudnicky",
    title     = "Overview of Robust and Multilingual Automatic Evaluation Metrics for Open-Domain Dialogue Systems at DSTC 11 Track 4",
    booktitle = "DSTC11: The Eleventh Dialog System Technology Challenge",
    series    = "24th Meeting of the Special Interest Group on Discourse and Dialogue (SIGDIAL)",
    year      = 2023,
    month     = "September",
    address   = "Prague, Czechia"
}

@inproceedings{li2017dailydialog,
  title     = "{D}aily{D}ialog: A Manually Labelled Multi-turn Dialogue Dataset",
  author    = "Li, Yanran and Su, Hui and Shen, Xiaoyu and Li, Wenjie and Cao, Ziqiang and Niu, Shuzi",
  booktitle = "Proceedings of the Eighth International Joint Conference on Natural Language Processing (Volume 1: Long Papers)",
  publisher = "Asian Federation of Natural Language Processing",
  address   = "Taipei, Taiwan",
  month     = nov,
  year      = {2017},
  pages     = "986--995",
  url       = "https://aclanthology.org/I17-1099/"
}

@inproceedings{zhang2018personalizing,
  title     = "Personalizing Dialogue Agents: {I} have a dog, do you have pets too?",
  author    = "Zhang, Saizheng and Dinan, Emily and Urbanek, Jack and Szlam, Arthur and Kiela, Douwe and Weston, Jason",
  booktitle = "Proceedings of the 56th Annual Meeting of the Association for Computational Linguistics (Volume 1: Long Papers)",
  publisher = "Association for Computational Linguistics",
  address   = "Melbourne, Australia",
  month     = jul,
  year      = {2018},
  pages     = "2204--2213",
  doi       = "10.18653/v1/P18-1205",
  url       = "https://aclanthology.org/P18-1205/"
}

@inproceedings{rashkin2019empathetic,
  title     = "Towards Empathetic Open-domain Conversation Models: A New Benchmark and Dataset",
  author    = "Rashkin, Hannah and Smith, Eric Michael and Li, Margaret and Boureau, Y-Lan",
  booktitle = "Proceedings of the 57th Annual Meeting of the Association for Computational Linguistics",
  publisher = "Association for Computational Linguistics",
  address   = "Florence, Italy",
  month     = jul,
  year      = {2019},
  pages     = "5370--5381",
  doi       = "10.18653/v1/P19-1534",
  url       = "https://aclanthology.org/P19-1534/"
}

@inproceedings{gopalakrishnan2019topical,
  title     = "Topical-Chat: Towards Knowledge-Grounded Open-Domain Conversations",
  author    = "Gopalakrishnan, Karthik and Hedayatnia, Behnam and Chen, Qinlang and Gottardi, Anna and Kwatra, Sanjeev and Venkatesh, Anu and Gabriel, Raefer and Hakkani-T{\"u}r, Dilek",
  booktitle = "INTERSPEECH",
  year      = {2019},
  pages     = "1891--1895",
  address   = "Graz, Austria",
  doi       = "10.21437/Interspeech.2019-3079",
  url       = "https://www.isca-archive.org/interspeech_2019/gopalakrishnan19_interspeech.html"
}

@incollection{wang2020lccc,
  title     = "A Large-Scale Chinese Short-Text Conversation Dataset",
  author    = "Wang, Yida and Ke, Pei and Zheng, Yinhe and Huang, Kaili and Jiang, Yong and Zhu, Xiaoyan and Huang, Minlie",
  booktitle = "Natural Language Processing and Chinese Computing: NLPCC 2020, Proceedings, Part I",
  series    = "Lecture Notes in Computer Science",
  volume    = "12430",
  publisher = "Springer, Cham",
  year      = {2020},
  pages     = "91--103",
  doi       = "10.1007/978-3-030-60450-9_8",
  url       = "https://link.springer.com/book/10.1007/978-3-030-60450-9"
}

@inproceedings{zhou2020kdconv,
  title     = "{K}d{C}onv: A {C}hinese Multi-domain Dialogue Dataset Towards Multi-turn Knowledge-driven Conversation",
  author    = "Zhou, Hao and Zheng, Chujie and Huang, Kaili and Huang, Minlie and Zhu, Xiaoyan",
  booktitle = "Proceedings of the 58th Annual Meeting of the Association for Computational Linguistics",
  publisher = "Association for Computational Linguistics",
  year      = {2020},
  address   = "Online",
  pages     = "7098--7108",
  doi       = "10.18653/v1/2020.acl-main.635",
  url       = "https://aclanthology.org/2020.acl-main.635/"
}

\appendix
\section{Appendix}

\subsection{$f$-Mutual Information}

\begin{definition}[$f$-Mutual Information \citep{kong2019information}]\label{def:fmif}
Given two random variables $X,Y$, the $f$-mutual information between $X$ and $Y$ is defined as
\begin{align*}
    \mi^f(X;Y)=\sum_{x,y}\Pr(x)\Pr(y)f\left(\frac{\Pr(x,y)}{\Pr(x)\Pr(y)}\right),
\end{align*}

where $f$ is a convex function and $f(1)=0$.
\end{definition}

By picking $f(t)=t\log t$, we obtain the classic Shannon mutual information,
\[
    \mi(X;Y)=\sum_{x,y}\Pr(x, y)\log \left(\frac{\Pr(x,y)}{\Pr(x)\Pr(y)}\right).
\]

$f$-Mutual Information has some desired properties.

\begin{lemma}[Properties of $f$-Mutual Information \citep{kong2019information}]
\label{properties}
$f$-Mutual information satisfies
\begin{description}
\item [Symmetry:] $\mi^f(X;Y)=\mi^f(Y;X)$;
\item [Non-negativity:] $\mi^f(X;Y)$ is always non-negative and is 0 if $X$ is independent of $Y$;
\item [Information Monotonicity:] $\mi^f(T(X);Y)\leq \mi^f(X;Y)$ where $T(\cdot)\in \mathbb{R}^{|\Sigma_X|\times |\Sigma_X|}$ is a possibly random operator on $X$ whose randomness is independent of $Y$.
\end{description}
\end{lemma}

Now we introduce $f$-divergence, Fenchel's duality and their relation to $f$-mutual information.  

% Then we give a formal definition to $f$-Modularity with these main technical ingredients.

\begin{definition}[Fenchel Duality \citep{rockafellar2015convex}]\label{def:dual}
Given any function $f:\mathbb{R}\mapsto \mathbb{R}$, we define its convex conjugate $f^{\star}$ as a function that also maps $\mathbb{R}$ to $\mathbb{R}$ such that $$f^{\star}(x)=\sup_{t} tx-f(t).$$
\end{definition}

\begin{lemma}[Dual Form of $f$-Divergence\cite{nguyen2010estimating}]\label{lemma:fdual-ineq}
\begin{align*}
    d_f(\mathbf{p};\mathbf{q}) \geq & \sup_{u \in \mathcal{U}} \sum_{\sigma\in \Sigma}u(\sigma) \mathbf{p}(\sigma)- \sum_{\sigma\in \Sigma}f^{\star}(u(\sigma))\mathbf{q}(\sigma) \\
    = & \sup_{u \in \mathcal{U}} \E_{\mathbf{p}} u- \E_{\mathbf{q}}f^{\star}(u),
\end{align*}
where $\mathcal{U}$ is a set of functions that maps $\Sigma$ to $\mathbb{R}$. The equality holds if and only if $u(\sigma)\in \partial{f}(\frac{\mathbf{p}(\sigma)}{\mathbf{q}(\sigma)})$, i.e., the subdifferential of $f$ on value $\frac{\mathbf{p}(\sigma)}{\mathbf{q}(\sigma)}$.  
\end{lemma}

\begin{definition}[$f$-Divergence \citep{ali1966general}]\label{def:df}
$f$-Divergence $d_f:\Delta_{\Sigma}\times \Delta_{\Sigma}\mapsto \mathbb{R}$ is a non-symmetric measure of the difference between distribution $\mathbf{p}\in \Delta_{\Sigma} $ and distribution $\mathbf{q}\in \Delta_{\Sigma} $
and is defined to be $$d_f(\mathbf{p};\mathbf{q})=\sum_{\sigma\in \Sigma}
\mathbf{q}(\sigma)f\bigg( \frac{\mathbf{p}(\sigma)}{\mathbf{q}(\sigma)}\bigg),$$
where $f:\mathbb{R}\mapsto\mathbb{R}$ is a convex function and $f(1)=0$.
\end{definition}

As an example, by picking $f(t)=t\log(t)$, we obtain the well-known KL-divergence $d_{KL}(\mathbf{p},\mathbf{q})=\sum_{\sigma}\mathbf{p}(\sigma)\log\frac{\mathbf{p}(\sigma)}{\mathbf{q}(\sigma)}$.

%Two common $f$-Divergences are KL Divergence $d_{KL}(\mathbf{p},\mathbf{q})=\sum_{\sigma}\mathbf{p}(\sigma)\log\frac{\mathbf{p}(\sigma)}{\mathbf{q}(\sigma)}$ by picking $f(t)=t\log(t)$, and Total Variance Distance $d_{TVD}(\mathbf{p},\mathbf{q})=\sum_{\sigma}|\mathbf{p}(\sigma)-\mathbf{q}(\sigma)|$ by picking $f(t)=|t-1|$.

The $f$-mutual information can be understood as the $f$-divergence between the joint distribution $\Pr(X=x,Y=y)$, denoted as $XY$, and the product of the marginal distribution $\Pr(X=x)\Pr(Y=y)$, denoted as $X\otimes Y$. 

\begin{lemma} We have the equation
    $$\mi^f(X;Y)=d_f(XY;X\otimes Y).$$
\end{lemma}

Now we introduce the dual form of $f$-divergence. Due to limited space, we defer the details of duality to the appendix.

\begin{lemma}[Dual Form of $f$-Divergence \citep{nguyen2010estimating}]\label{lemma:fdual}
\[
d_f(\mathbf{p};\mathbf{q}) = \sup_{D\in \mathcal{D}} \E_{\mathbf{p}} \partial{f}(D)- \E_{\mathbf{q}}f^{\star}(\partial{f}(D))
\]
where $\mathcal{D}$ is a set of functions that maps $\Sigma$ to $\mathbb{R}$ and the best $D^*$ satisfies $D^*(\sigma)=\frac{\mathbf{p}(\sigma)}{\mathbf{q}(\sigma)}$. $f^*$ is the Fenchel Duality of $f$.
\end{lemma}

Some common $f$ functions for $f$-divergence and their dual forms are shown in ~\cref{table:distinguishers}.

\begin{table*}[htp]
\caption{Reference Table for $f,f^{\star}$ \protect\citep{kong2019information} }
\begin{center}
\begin{tabular}{llll}
\toprule
    {$f$-Divergence} & {$f(t)$} & {$\partial{f}(D)$} & {$f^{\star}(\partial{f}(D)$)} \\ \midrule\midrule
    Total Variation Distance  & $|t-1|$  & $\sign(\log D)$ & $\sign(\log D)$ \\ 
    \midrule
    KL-Divergence & $t\log t$  & $\log D + 1$ & $D$ \\ 
    \midrule
    Pearson $\chi^2$ & $(t-1)^2$  & $2(D-1)$ & $D^2-1$ \\ 
    \midrule
    Jensen-Shannon & $-(t+1)\log{\frac{t+1}{2}}+t\log t$  & $\log{\frac{2D}{1+D}}$ & $-\log(\frac{2}{1+D})$ \\
    \midrule
    Squared Hellinger &$(\sqrt{t}-1)^2$ & $1-\sqrt{\frac{1}{D}}$ & $\sqrt{D}-1$ \\
    \bottomrule
\end{tabular}
\end{center}
\label{table:distinguishers}
\end{table*}

\section{Omitted Proofs}
\label{apx:proof}
\begin{proof}[Proof of ~\Cref{thm:convergence}]

    We prove the theorem by two lemmas from \cite{belghazi2018mutual}. 
    \begin{lemma}[universal approximation \citep{belghazi2018mutual}]
        For any $\epsilon>0$, there exists a neural network parametered by $\theta\in\Theta$ to approximate the optimal mutual information. Suppose \\ $\sup_{\theta\in\Theta} -L_\mi(\theta)=\mi(\Theta)$, then
            $$|\mi(\Theta)-MI(C,R)|<\epsilon, a.e.$$
    \end{lemma}

    \begin{lemma}[estimation \citep{belghazi2018mutual}]
    Suppose \\ $\sup_{\theta\in\Theta} -L_{\mi,D^(n)}(\theta)=\mi_n(\Theta)$. For any $\epsilon>0$, there exists an integer $N$ such that for any $n>N$, 
        $$|\mi_n(\Theta)-\mi(\Theta)|<\epsilon, a.e.,$$
        where $-L_{\mi,D^{(n)}}(\theta)$ is the loss function defined in the dataset $D^{(n)}$ with size $n$.
    \end{lemma}

    Combining these lemmas and using the triangle inequality, we have     
    $$\E[\RelScore(c,r)]\;\longrightarrow\; \mi(C,R).$$

    Since $\E[\RelScore(c,r)]\le\mi(C,R)$ and the equation holds only if $\RelScore(c,r)=\pmi(c,r)$. Furthermore, we have

    $$\RelScore(c,r)\;\longrightarrow\; \pmi(c,r).$$
\end{proof}
\section{Synthetic Study: Additional Details and Results}
\label{app:synthetic}
\subsection{Additional Details}
\paragraph{Datasets.}
We study three synthetic corpora that impose different dependency structures over context–response pairs: \emph{Diagonal}, \emph{Block}, and \emph{Independent}.  
For each dataset we predefine \textbf{20} context prototypes $\{\tilde c_i\}_{i=1}^{20}$ and \textbf{20} response prototypes $\{\tilde r_j\}_{j=1}^{20}$ (see code’s \texttt{prototypes} for concrete strings). A dataset-specific joint distribution over indices $(i,j)\in\{1,\ldots,20\}^2$ governs sampling:
\begin{itemize}
\item \textbf{Diagonal:} mass concentrated on $i=j$ (approximately one-to-one alignment);
\item \textbf{Independent:} product distribution $\pi(i,j)=\pi_X(i)\pi_Y(j)$ (no dependency);
\item \textbf{Block:} block-structured coupling between topical groups of contexts and generic supportive responses.
\end{itemize}
Each sampled prototype pair $(\tilde c_i,\tilde r_j)$ is then \textbf{paraphrased by an LLM} to produce $(c,r)$ to increase diversity; \emph{all} embeddings and training use the paraphrased $(c,r)$ (not the raw prototypes).  
This enlarges the support, better matches real text diversity, and reduces overfitting.  
Negatives are created by random mismatching within the same dataset; each positive is paired with \textbf{four} negatives.

\paragraph{Encoders and Representations.}
Each pair $(c,r)$ is encoded via a single prompt that forms one input sequence for a frozen LLM encoder. We use the following prompt verbatim:
\begin{quote}\small\ttfamily\raggedright
You are an assistant skilled at evaluating the relevance of a response to a given context.\\
Task: Evaluate the relevance of the following response to the context.\\
Context: \{context\_text\}\\
Response: \{response\_text\}\\
Result:
\end{quote}
The encoder processes this prompt (with \texttt{\{context\_text\}} and \texttt{\{response\_text\}} filled by $c$ and $r$) and outputs a single representation vector $\phi(c,r)$ for the \emph{pair}.  
All embeddings are produced locally (vLLM). Encoders: \texttt{Llama-3.2-3B-Instruct},\\ \texttt{Phi-4-mini-instruct}, \texttt{Qwen3-\{0.6B, 1.7B, 4B, 8B\}}.

\paragraph{Scoring Network.}
On top of the frozen encoder, we train a compact MLP head that maps $\phi(c,r)$ to a score $s_\theta(c,r)$:
\begin{align*}
&\text{Linear}(d, 256) \to \text{PReLU} \\
\to &\text{Linear}(256, 128) \to \text{PReLU}\\
\to &\text{Linear}(128, 1) \to \text{softcap}_{20}
\end{align*}

where $\text{softcap}_{20}(x)=20\tanh(x/20)$. 

\paragraph{Training Setup.}
We use AdamW. A unified rule sets the learning rate for \emph{all} encoders:
\[
\text{lr} \;=\; 1\times10^{-3}\times \frac{1024}{d},
\]
with $d$ the encoder embedding dimension. We train for \textbf{100} epochs with positive-batch size \textbf{256} and \textbf{4} negatives per positive ($K{=}4$). Model selection uses validation correlation against synthetic ground-truth PMI.

\paragraph{Objective Functions.}
Let $\mathcal{P}_+$ and $\mathcal{P}_-$ be the positive (true) and negative (mismatched) pair distributions. The model outputs $s_\theta(c,r)$ on the prompt-based pair embedding $\phi(c,r)$. Consistent with the implementation:
\begin{itemize}
    \item \textbf{PMIScore}
    \[
    \mathcal{L}_{\text{PMI}} \;=\; -\Big(\E_{(c,r)\sim\mathcal{P}_+}[s_\theta(c,r)] \;-\; \E_{(c,r^-)\sim\mathcal{P}_-}\!\big[\exp(s_\theta(c,r^-))\big]\Big).
    \]
    \item \textbf{InfoNCE}
    \[
    \mathcal{L}_{\text{NCE}} \;=\; -\,\E\!\left[\log \frac{\exp(s_\theta(c,r_0))}{\exp(s_\theta(c,r_0))+\sum_{k=1}^{K}\exp(s_\theta(c,r_k^-))}\right].
    \]
    \item \textbf{MINE}
    \[
    \mathcal{L}_{\text{MINE}} \;=\; -\Big(\E_{(c,r)\sim\mathcal{P}_+}[s_\theta(c,r)] \;-\; \log \E_{(c,r^-)\sim\mathcal{P}_-}\!\big[\exp(s_\theta(c,r^-))\big]\Big).
    \]
    \item \textbf{KDE} is a non-parametric baseline: 
    \[\mathrm{KDE}(c,r)=\log \hat p(c,r)-\log \hat p(c)\hat p(r)\] with Gaussian kernels and cross-validated bandwidths. At evaluation we use $\mathrm{KDE}$ as the score.
\end{itemize}

\begin{table*}[!ht]
\centering
\small
\setlength{\tabcolsep}{3.5pt}
\resizebox{\textwidth}{!}{%
\begin{tabular}{llcccccccccc}
\toprule
\multirow{2}{*}{Dataset} & \multirow{2}{*}{Model} & \multicolumn{2}{c}{PMIScore} & \multicolumn{2}{c}{MINE} & \multicolumn{2}{c}{InfoNCE} & \multicolumn{2}{c}{KDE} & \multicolumn{2}{c}{MEEP} \\
 & & $\rho$ & MSE & $\rho$ & MSE & $\rho$ & MSE & $\rho$ & MSE & $\rho$ & MSE \\
\midrule
\multirow{9}{*}{Block}
 & Phi-4-mini-instruct & \makecell{\textbf{0.802} \\[-3pt] \scriptsize (0.017)} & \makecell{\textbf{0.350} \\[-3pt] \scriptsize (0.070)} & \makecell{0.754 \\[-3pt] \scriptsize (0.023)} & \makecell{2.520 \\[-3pt] \scriptsize (3.499)} & \makecell{0.662 \\[-3pt] \scriptsize (0.042)} & \makecell{1.686 \\[-3pt] \scriptsize (1.042)} & \makecell{0.702 \\[-3pt] \scriptsize (0.016)} & \makecell{1.923 \\[-3pt] \scriptsize (0.036)} & \makecell{0.124 \\[-3pt] \scriptsize (0.021)} & \makecell{2.902 \\[-3pt] \scriptsize (0.038)}\\
 & Qwen3-0.6B & \makecell{0.470 \\[-3pt] \scriptsize (0.019)} & \makecell{\textbf{1.287} \\[-3pt] \scriptsize (0.424)} & \makecell{\textbf{0.533} \\[-3pt] \scriptsize (0.012)} & \makecell{14.253 \\[-3pt] \scriptsize (16.473)} & \makecell{0.370 \\[-3pt] \scriptsize (0.016)} & \makecell{12.089 \\[-3pt] \scriptsize (13.000)} & \makecell{0.390 \\[-3pt] \scriptsize (0.009)} & \makecell{4.268 \\[-3pt] \scriptsize (0.019)} & \makecell{-0.002 \\[-3pt] \scriptsize (0.043)} & \makecell{3.119 \\[-3pt] \scriptsize (0.079)}\\
 & Qwen3-1.7B & \makecell{\textbf{0.649} \\[-3pt] \scriptsize (0.038)} & \makecell{\textbf{0.692} \\[-3pt] \scriptsize (0.331)} & \makecell{0.578 \\[-3pt] \scriptsize (0.019)} & \makecell{12.457 \\[-3pt] \scriptsize (8.151)} & \makecell{0.464 \\[-3pt] \scriptsize (0.027)} & \makecell{5.630 \\[-3pt] \scriptsize (3.639)} & \makecell{0.550 \\[-3pt] \scriptsize (0.021)} & \makecell{2.677 \\[-3pt] \scriptsize (0.028)} & \makecell{0.037 \\[-3pt] \scriptsize (0.042)} & \makecell{3.058 \\[-3pt] \scriptsize (0.069)}\\
 & Qwen3-4B & \makecell{\textbf{0.777} \\[-3pt] \scriptsize (0.017)} & \makecell{\textbf{0.371} \\[-3pt] \scriptsize (0.064)} & \makecell{0.742 \\[-3pt] \scriptsize (0.029)} & \makecell{2.080 \\[-3pt] \scriptsize (2.841)} & \makecell{0.639 \\[-3pt] \scriptsize (0.022)} & \makecell{0.850 \\[-3pt] \scriptsize (0.053)} & \makecell{0.638 \\[-3pt] \scriptsize (0.009)} & \makecell{2.059 \\[-3pt] \scriptsize (0.061)} & \makecell{0.096 \\[-3pt] \scriptsize (0.045)} & \makecell{2.964 \\[-3pt] \scriptsize (0.075)}\\
 & Qwen3-8B & \makecell{\textbf{0.811} \\[-3pt] \scriptsize (0.010)} & \makecell{\textbf{0.337} \\[-3pt] \scriptsize (0.047)} & \makecell{0.808 \\[-3pt] \scriptsize (0.019)} & \makecell{0.398 \\[-3pt] \scriptsize (0.075)} & \makecell{0.750 \\[-3pt] \scriptsize (0.011)} & \makecell{0.702 \\[-3pt] \scriptsize (0.262)} & \makecell{0.632 \\[-3pt] \scriptsize (0.024)} & \makecell{2.533 \\[-3pt] \scriptsize (0.104)} & \makecell{0.191 \\[-3pt] \scriptsize (0.035)} & \makecell{2.789 \\[-3pt] \scriptsize (0.077)}\\
 & Llama-3.2-3B-Instruct & \makecell{\textbf{0.771} \\[-3pt] \scriptsize (0.017)} & \makecell{\textbf{0.404} \\[-3pt] \scriptsize (0.106)} & \makecell{0.752 \\[-3pt] \scriptsize (0.014)} & \makecell{1.055 \\[-3pt] \scriptsize (1.130)} & \makecell{0.639 \\[-3pt] \scriptsize (0.061)} & \makecell{3.618 \\[-3pt] \scriptsize (3.860)} & \makecell{0.664 \\[-3pt] \scriptsize (0.017)} & \makecell{1.986 \\[-3pt] \scriptsize (0.065)} & \makecell{0.128 \\[-3pt] \scriptsize (0.027)} & \makecell{2.893 \\[-3pt] \scriptsize (0.050)}\\
\midrule
\multirow{9}{*}{Diagonal}
 & Phi-4-mini-instruct & \makecell{\textbf{0.664} \\[-3pt] \scriptsize (0.009)} & \makecell{\textbf{1.728} \\[-3pt] \scriptsize (0.251)} & \makecell{0.618 \\[-3pt] \scriptsize (0.010)} & \makecell{44.496 \\[-3pt] \scriptsize (11.139)} & \makecell{0.561 \\[-3pt] \scriptsize (0.044)} & \makecell{2.786 \\[-3pt] \scriptsize (1.086)} & \makecell{0.609 \\[-3pt] \scriptsize (0.021)} & \makecell{5.167 \\[-3pt] \scriptsize (0.095)} & \makecell{0.120 \\[-3pt] \scriptsize (0.034)} & \makecell{5.781 \\[-3pt] \scriptsize (0.083)}\\
 & Qwen3-0.6B & \makecell{\textbf{0.525} \\[-3pt] \scriptsize (0.012)} & \makecell{\textbf{3.210} \\[-3pt] \scriptsize (0.934)} & \makecell{0.518 \\[-3pt] \scriptsize (0.025)} & \makecell{58.645 \\[-3pt] \scriptsize (91.388)} & \makecell{0.387 \\[-3pt] \scriptsize (0.026)} & \makecell{8.500 \\[-3pt] \scriptsize (6.521)} & \makecell{0.381 \\[-3pt] \scriptsize (0.010)} & \makecell{8.336 \\[-3pt] \scriptsize (0.041)} & \makecell{0.021 \\[-3pt] \scriptsize (0.022)} & \makecell{6.021 \\[-3pt] \scriptsize (0.054)}\\
 & Qwen3-1.7B & \makecell{\textbf{0.608} \\[-3pt] \scriptsize (0.008)} & \makecell{\textbf{2.545} \\[-3pt] \scriptsize (0.762)} & \makecell{0.589 \\[-3pt] \scriptsize (0.021)} & \makecell{52.908 \\[-3pt] \scriptsize (11.390)} & \makecell{0.545 \\[-3pt] \scriptsize (0.030)} & \makecell{9.884 \\[-3pt] \scriptsize (8.883)} & \makecell{0.392 \\[-3pt] \scriptsize (0.009)} & \makecell{6.756 \\[-3pt] \scriptsize (0.086)} & \makecell{0.048 \\[-3pt] \scriptsize (0.024)} & \makecell{5.926 \\[-3pt] \scriptsize (0.082)}\\
 & Qwen3-4B & \makecell{\textbf{0.629} \\[-3pt] \scriptsize (0.007)} & \makecell{\textbf{1.942} \\[-3pt] \scriptsize (0.257)} & \makecell{0.611 \\[-3pt] \scriptsize (0.014)} & \makecell{26.183 \\[-3pt] \scriptsize (9.451)} & \makecell{0.549 \\[-3pt] \scriptsize (0.015)} & \makecell{4.612 \\[-3pt] \scriptsize (1.841)} & \makecell{0.543 \\[-3pt] \scriptsize (0.007)} & \makecell{5.642 \\[-3pt] \scriptsize (0.080)} & \makecell{0.050 \\[-3pt] \scriptsize (0.035)} & \makecell{5.948 \\[-3pt] \scriptsize (0.161)}\\
 & Qwen3-8B & \makecell{\textbf{0.659} \\[-3pt] \scriptsize (0.008)} & \makecell{\textbf{1.769} \\[-3pt] \scriptsize (0.207)} & \makecell{0.637 \\[-3pt] \scriptsize (0.007)} & \makecell{7.802 \\[-3pt] \scriptsize (3.534)} & \makecell{0.598 \\[-3pt] \scriptsize (0.036)} & \makecell{2.397 \\[-3pt] \scriptsize (1.459)} & \makecell{0.594 \\[-3pt] \scriptsize (0.010)} & \makecell{5.115 \\[-3pt] \scriptsize (0.114)} & \makecell{0.321 \\[-3pt] \scriptsize (0.029)} & \makecell{5.231 \\[-3pt] \scriptsize (0.067)}\\
 & Llama-3.2-3B-Instruct & \makecell{\textbf{0.647} \\[-3pt] \scriptsize (0.017)} & \makecell{\textbf{2.093} \\[-3pt] \scriptsize (0.535)} & \makecell{0.634 \\[-3pt] \scriptsize (0.017)} & \makecell{12.449 \\[-3pt] \scriptsize (2.647)} & \makecell{0.598 \\[-3pt] \scriptsize (0.028)} & \makecell{2.400 \\[-3pt] \scriptsize (0.456)} & \makecell{0.590 \\[-3pt] \scriptsize (0.004)} & \makecell{5.028 \\[-3pt] \scriptsize (0.129)} & \makecell{0.200 \\[-3pt] \scriptsize (0.021)} & \makecell{5.578 \\[-3pt] \scriptsize (0.053)}\\
\midrule
\multirow{9}{*}{Independent}
 & Phi-4-mini-instruct & \makecell{\textbf{0.110} \\[-3pt] \scriptsize (0.009)} & \makecell{\textbf{0.255} \\[-3pt] \scriptsize (0.035)} & \makecell{0.098 \\[-3pt] \scriptsize (0.014)} & \makecell{27.308 \\[-3pt] \scriptsize (52.609)} & \makecell{0.079 \\[-3pt] \scriptsize (0.026)} & \makecell{5.817 \\[-3pt] \scriptsize (6.019)} & \makecell{0.014 \\[-3pt] \scriptsize (0.011)} & \makecell{3.334 \\[-3pt] \scriptsize (0.130)} & \makecell{-0.009 \\[-3pt] \scriptsize (0.029)} & \makecell{1.227 \\[-3pt] \scriptsize (0.026)}\\
 & Qwen3-0.6B & \makecell{0.018 \\[-3pt] \scriptsize (0.034)} & \makecell{\textbf{0.228} \\[-3pt] \scriptsize (0.009)} & \makecell{\textbf{0.061} \\[-3pt] \scriptsize (0.013)} & \makecell{11.448 \\[-3pt] \scriptsize (9.988)} & \makecell{0.034 \\[-3pt] \scriptsize (0.017)} & \makecell{23.007 \\[-3pt] \scriptsize (40.962)} & \makecell{0.021 \\[-3pt] \scriptsize (0.019)} & \makecell{1.479 \\[-3pt] \scriptsize (0.007)} & \makecell{-0.008 \\[-3pt] \scriptsize (0.013)} & \makecell{1.231 \\[-3pt] \scriptsize (0.011)}\\
 & Qwen3-1.7B & \makecell{0.076 \\[-3pt] \scriptsize (0.027)} & \makecell{\textbf{0.252} \\[-3pt] \scriptsize (0.013)} & \makecell{\textbf{0.128} \\[-3pt] \scriptsize (0.032)} & \makecell{21.328 \\[-3pt] \scriptsize (16.693)} & \makecell{0.045 \\[-3pt] \scriptsize (0.033)} & \makecell{7.844 \\[-3pt] \scriptsize (11.021)} & \makecell{0.001 \\[-3pt] \scriptsize (0.021)} & \makecell{1.959 \\[-3pt] \scriptsize (0.032)} & \makecell{-0.006 \\[-3pt] \scriptsize (0.029)} & \makecell{1.223 \\[-3pt] \scriptsize (0.025)}\\
 & Qwen3-4B & \makecell{0.046 \\[-3pt] \scriptsize (0.018)} & \makecell{\textbf{0.256} \\[-3pt] \scriptsize (0.006)} & \makecell{\textbf{0.051} \\[-3pt] \scriptsize (0.017)} & \makecell{19.034 \\[-3pt] \scriptsize (25.263)} & \makecell{0.033 \\[-3pt] \scriptsize (0.016)} & \makecell{3.188 \\[-3pt] \scriptsize (3.290)} & \makecell{0.029 \\[-3pt] \scriptsize (0.018)} & \makecell{1.699 \\[-3pt] \scriptsize (0.018)} & \makecell{-0.025 \\[-3pt] \scriptsize (0.016)} & \makecell{1.242 \\[-3pt] \scriptsize (0.018)}\\
 & Qwen3-8B & \makecell{0.068 \\[-3pt] \scriptsize (0.027)} & \makecell{\textbf{0.253} \\[-3pt] \scriptsize (0.008)} & \makecell{\textbf{0.089} \\[-3pt] \scriptsize (0.030)} & \makecell{5.926 \\[-3pt] \scriptsize (7.212)} & \makecell{0.015 \\[-3pt] \scriptsize (0.023)} & \makecell{3.545 \\[-3pt] \scriptsize (1.001)} & \makecell{0.020 \\[-3pt] \scriptsize (0.021)} & \makecell{4.732 \\[-3pt] \scriptsize (0.186)} & \makecell{-0.027 \\[-3pt] \scriptsize (0.025)} & \makecell{1.246 \\[-3pt] \scriptsize (0.025)}\\
 & Llama-3.2-3B-Instruct & \makecell{0.039 \\[-3pt] \scriptsize (0.023)} & \makecell{\textbf{0.252} \\[-3pt] \scriptsize (0.010)} & \makecell{\textbf{0.119} \\[-3pt] \scriptsize (0.024)} & \makecell{3.357 \\[-3pt] \scriptsize (1.649)} & \makecell{0.036 \\[-3pt] \scriptsize (0.015)} & \makecell{3.571 \\[-3pt] \scriptsize (4.536)} & \makecell{0.022 \\[-3pt] \scriptsize (0.016)} & \makecell{2.617 \\[-3pt] \scriptsize (0.099)} & \makecell{0.013 \\[-3pt] \scriptsize (0.022)} & \makecell{1.209 \\[-3pt] \scriptsize (0.020)}\\
\bottomrule
\end{tabular}
}
\caption{Performance on Synthetic datasets. Best Pearson $\rho$ (higher) and MSE (lower) per model are \textbf{bolded}.}
\label{tab:synthetic}
\end{table*}

\subsection{Additional Results}
\paragraph{Metrics.}
We report Spearman’s $\rho$ (higher is better) and mean squared error (MSE; lower is better) between predicted scores and ground-truth synthetic PMI.

\paragraph{Results.}
Table~\ref{tab:synthetic} shows results for the four methods by dataset and encoder. Within each model group, the best $\rho$ (max) and MSE (min) are \textbf{boldfaced}. PMIScore is generally strongest, especially on \emph{Diagonal} and \emph{Block}.
\section{Empirical Study: Additional Details and Results}
\label{app:empirical}

\subsection{Additional Details}

\paragraph{Datasets and splits.}
We evaluate on the multilingual dialogue collections from DSTC-11 Track~4 (English \texttt{dstc\_en}, Chinese \texttt{dstc\_zh}; see \S\ref{sec:empirical}). 
Per language we draw $3{,}000/1{,}000/1{,}000$ context--response pairs for \emph{train}/\emph{validation}/\emph{test} from the Track~4 training pool, fixing a global seed ($42$) for reproducibility. 
Model selection uses \emph{validation AUC}; all retrieval numbers reported in the main text are on the held-out \emph{test} split.
Human correlation is computed only on the official human-annotated dev packs (English $n{=}375$, Chinese $n{=}1{,}589$).

\paragraph{Negative sampling.}
Each positive $(c,r^+)$ is paired with \textbf{four} negatives: \textbf{1 in-dialogue} negative (pair $c$ with a different turn from the \emph{same} dialogue that is not the gold continuation) plus \textbf{3 random} negatives (responses sampled from the same split but \emph{outside} the current dialogue). 
In-dialogue candidates are drawn only from responses observed within that dialogue in the same split to avoid leakage across splits. 
Random negatives are sampled without replacement within a mini-batch; overall class ratio is thus $1{:}4$ (pos{:}neg).

\paragraph{Preprocessing of contexts and responses.}
We keep multi-turn histories. For each dialogue with turns $\{u_1,\dots,u_i,\dots\}$, a positive pair is formed as
$\text{context}=u_1\!\oplus\!\cdots\!\oplus\!u_{i-1}$ (newline-joined, no speaker tags)
and $\text{response}=u_i$ (stripped of any speaker prefixes).%
\footnote{The loader normalizes common JSON/CSV variants for contexts (for example, \texttt{context}, \texttt{history}, and \texttt{dialog}) and responses (for example, \texttt{response}, \texttt{reference}, and \texttt{answer}). It also converts lists to text, removes empty lines, and trims whitespace.}
This yields realistic, variable-length contexts and naturally hard in-dialogue negatives.

\paragraph{Prompted pair encoding (frozen encoders).}
We encode each $(c,r)$ with a single \emph{pair-level} prompt to a frozen LLM encoder, producing one vector $\phi(c,r)$ per pair:
\begin{quote}\small\ttfamily\raggedright
You are an assistant skilled at evaluating the relevance of a response to a given context.\\
Task: Evaluate the relevance of the following response to the context.\\
Context: \{context\_text\}\\
Response: \{response\_text\}\\
Result:
\end{quote}
Unless otherwise noted, we use the \emph{English} template for \emph{both} languages for consistency.\footnote{A Chinese template is available in code but disabled by default.} 
We use vLLM (\texttt{task=embed}) to process the embeddings.

\paragraph{LLMs (frozen).}
Encoders: \texttt{Qwen3-\{0.6B,1.7B,4B,8B\}}, \\ \texttt{Llama-3.2-3B-Instruct},\\ \texttt{microsoft/Phi-4-mini-instruct}. 
We only train the small scorer head; encoders remain frozen throughout.

\paragraph{Scorer head (shared across methods).}
A lightweight MLP maps $\phi(c,r)\in\mathbb{R}^d$ to $s_\theta(c,r)\in\mathbb{R}$, the structure is the same as the MLP in synthetic study.

\paragraph{Objectives.}
Given positives $\mathcal{P}_+$ and (constructed) negatives $\mathcal{P}_-$, losses follow App.~\ref{app:synthetic}: 
PMIScore (dual-form), InfoNCE (batch-wise contrastive), MINE (DV bound). 
KDE is non-parametric: we fit StandardScaler~$\to$ optional PCA ($128$ dims)~$\to$ Gaussian KDE for $p_+$ and $p_-$ with an automatic bandwidth heuristic; the score is $\log p_+ - \log p_-$.

\paragraph{Optimization and early stopping.}
We use AdamW with learning rate $\text{lr}=10^{-3}\!\times\!(1024/d)$ where $d$ is the encoder embedding size. 
We train for up to $100$ epochs, with positive batch size $256$ and \texttt{NEG\_RATIO}$=4$ (so each \emph{positive} batch is matched with a \emph{negative} batch of size $4\times$ for InfoNCE). 
We select checkpoints by \emph{validation AUC} and report AUC on \emph{test}; Spearman is computed on human dev sets.

\paragraph{Metric computation details.}
\textbf{ROC-AUC} is computed over the binary labels (\texttt{is\_positive}) from the raw scores $s_\theta(c,r)$ (no thresholding). 
\textbf{Spearman’s $\rho$} is computed between $s_\theta(c,r)$ and the human relevance field \texttt{annot\_relevant\_mean}; ties are handled by the standard rank definition in \texttt{scipy}/\texttt{sklearn}. 
We report per-backbone results and also averages across backbones in the main figure.

\subsection{Complete Results (Test AUC)}
Table~\ref{tab:empirical} reports ROC-AUC on the response-ranking task for all methods, split by language and backbone (\emph{test only}).

\begin{table*}[!ht]
\centering
\small
\setlength{\tabcolsep}{3.5pt}
\resizebox{\textwidth}{!}{%
\begin{tabular}{llcccccccccc}
\toprule
\multirow{2}{*}{Dataset} & \multirow{2}{*}{Model} & \multicolumn{2}{c}{PMIScore} & \multicolumn{2}{c}{MINE} & \multicolumn{2}{c}{InfoNCE} & \multicolumn{2}{c}{KDE} & \multicolumn{2}{c}{MEEP} \\
 & & AUC & $\rho$ & AUC & $\rho$ & AUC & $\rho$ & AUC & $\rho$ & AUC & $\rho$ \\
\midrule
\multirow{9}{*}{English}
 & Phi-4-mini-instruct & \makecell{\textbf{0.922} \\[-3pt] \scriptsize (0.002)} & \makecell{0.407 \\[-3pt] \scriptsize (0.096)} & \makecell{0.910 \\[-3pt] \scriptsize (0.005)} & \makecell{0.379 \\[-3pt] \scriptsize (0.107)} & \makecell{0.919 \\[-3pt] \scriptsize (0.002)} & \makecell{0.436 \\[-3pt] \scriptsize (0.075)} & \makecell{0.881 \\[-3pt] \scriptsize (0.002)} & \makecell{\textbf{0.482} \\[-3pt] \scriptsize (0.052)} & \makecell{0.673 \\[-3pt] \scriptsize (0.010)} & \makecell{0.335 \\[-3pt] \scriptsize (0.085)}\\
 & Qwen3-0.6B & \makecell{\textbf{0.858} \\[-3pt] \scriptsize (0.006)} & \makecell{\textbf{0.286} \\[-3pt] \scriptsize (0.075)} & \makecell{0.829 \\[-3pt] \scriptsize (0.026)} & \makecell{0.258 \\[-3pt] \scriptsize (0.081)} & \makecell{0.830 \\[-3pt] \scriptsize (0.009)} & \makecell{0.175 \\[-3pt] \scriptsize (0.069)} & \makecell{0.781 \\[-3pt] \scriptsize (0.005)} & \makecell{0.214 \\[-3pt] \scriptsize (0.092)} & \makecell{0.531 \\[-3pt] \scriptsize (0.013)} & \makecell{0.129 \\[-3pt] \scriptsize (0.127)}\\
 & Qwen3-1.7B & \makecell{\textbf{0.903} \\[-3pt] \scriptsize (0.004)} & \makecell{\textbf{0.433} \\[-3pt] \scriptsize (0.055)} & \makecell{0.892 \\[-3pt] \scriptsize (0.003)} & \makecell{0.406 \\[-3pt] \scriptsize (0.092)} & \makecell{0.885 \\[-3pt] \scriptsize (0.006)} & \makecell{0.321 \\[-3pt] \scriptsize (0.078)} & \makecell{0.793 \\[-3pt] \scriptsize (0.004)} & \makecell{0.180 \\[-3pt] \scriptsize (0.060)} & \makecell{0.593 \\[-3pt] \scriptsize (0.018)} & \makecell{-0.018 \\[-3pt] \scriptsize (0.082)}\\
 & Qwen3-4B & \makecell{\textbf{0.941} \\[-3pt] \scriptsize (0.002)} & \makecell{\textbf{0.475} \\[-3pt] \scriptsize (0.028)} & \makecell{0.933 \\[-3pt] \scriptsize (0.003)} & \makecell{0.445 \\[-3pt] \scriptsize (0.024)} & \makecell{0.936 \\[-3pt] \scriptsize (0.006)} & \makecell{0.351 \\[-3pt] \scriptsize (0.090)} & \makecell{0.914 \\[-3pt] \scriptsize (0.003)} & \makecell{0.381 \\[-3pt] \scriptsize (0.076)} & \makecell{0.670 \\[-3pt] \scriptsize (0.016)} & \makecell{0.038 \\[-3pt] \scriptsize (0.110)}\\
 & Qwen3-8B & \makecell{\textbf{0.952} \\[-3pt] \scriptsize (0.003)} & \makecell{\textbf{0.467} \\[-3pt] \scriptsize (0.032)} & \makecell{0.947 \\[-3pt] \scriptsize (0.003)} & \makecell{0.462 \\[-3pt] \scriptsize (0.023)} & \makecell{0.952 \\[-3pt] \scriptsize (0.003)} & \makecell{0.414 \\[-3pt] \scriptsize (0.045)} & \makecell{0.926 \\[-3pt] \scriptsize (0.004)} & \makecell{0.365 \\[-3pt] \scriptsize (0.074)} & \makecell{0.798 \\[-3pt] \scriptsize (0.008)} & \makecell{0.173 \\[-3pt] \scriptsize (0.153)}\\
 & Llama-3.2-3B-Instruct & \makecell{\textbf{0.934} \\[-3pt] \scriptsize (0.003)} & \makecell{0.384 \\[-3pt] \scriptsize (0.069)} & \makecell{0.926 \\[-3pt] \scriptsize (0.005)} & \makecell{0.371 \\[-3pt] \scriptsize (0.058)} & \makecell{0.907 \\[-3pt] \scriptsize (0.009)} & \makecell{0.349 \\[-3pt] \scriptsize (0.104)} & \makecell{0.884 \\[-3pt] \scriptsize (0.002)} & \makecell{\textbf{0.393} \\[-3pt] \scriptsize (0.095)} & \makecell{0.653 \\[-3pt] \scriptsize (0.009)} & \makecell{0.105 \\[-3pt] \scriptsize (0.090)}\\
\midrule
\multirow{9}{*}{Chinese}
 & Phi-4-mini-instruct & \makecell{\textbf{0.824} \\[-3pt] \scriptsize (0.003)} & \makecell{\textbf{0.475} \\[-3pt] \scriptsize (0.047)} & \makecell{0.812 \\[-3pt] \scriptsize (0.006)} & \makecell{0.453 \\[-3pt] \scriptsize (0.045)} & \makecell{0.799 \\[-3pt] \scriptsize (0.006)} & \makecell{0.405 \\[-3pt] \scriptsize (0.037)} & \makecell{0.741 \\[-3pt] \scriptsize (0.002)} & \makecell{0.389 \\[-3pt] \scriptsize (0.037)} & \makecell{0.575 \\[-3pt] \scriptsize (0.011)} & \makecell{0.128 \\[-3pt] \scriptsize (0.060)}\\
 & Qwen3-0.6B & \makecell{0.774 \\[-3pt] \scriptsize (0.005)} & \makecell{\textbf{0.430} \\[-3pt] \scriptsize (0.023)} & \makecell{\textbf{0.776} \\[-3pt] \scriptsize (0.006)} & \makecell{0.404 \\[-3pt] \scriptsize (0.080)} & \makecell{0.761 \\[-3pt] \scriptsize (0.011)} & \makecell{0.388 \\[-3pt] \scriptsize (0.073)} & \makecell{0.682 \\[-3pt] \scriptsize (0.009)} & \makecell{0.170 \\[-3pt] \scriptsize (0.064)} & \makecell{0.531 \\[-3pt] \scriptsize (0.009)} & \makecell{0.080 \\[-3pt] \scriptsize (0.047)}\\
 & Qwen3-1.7B & \makecell{\textbf{0.812} \\[-3pt] \scriptsize (0.004)} & \makecell{\textbf{0.474} \\[-3pt] \scriptsize (0.056)} & \makecell{0.801 \\[-3pt] \scriptsize (0.006)} & \makecell{0.473 \\[-3pt] \scriptsize (0.058)} & \makecell{0.789 \\[-3pt] \scriptsize (0.002)} & \makecell{0.404 \\[-3pt] \scriptsize (0.043)} & \makecell{0.700 \\[-3pt] \scriptsize (0.004)} & \makecell{0.320 \\[-3pt] \scriptsize (0.023)} & \makecell{0.525 \\[-3pt] \scriptsize (0.011)} & \makecell{0.051 \\[-3pt] \scriptsize (0.057)}\\
 & Qwen3-4B & \makecell{\textbf{0.861} \\[-3pt] \scriptsize (0.004)} & \makecell{\textbf{0.552} \\[-3pt] \scriptsize (0.048)} & \makecell{0.852 \\[-3pt] \scriptsize (0.005)} & \makecell{0.535 \\[-3pt] \scriptsize (0.050)} & \makecell{0.846 \\[-3pt] \scriptsize (0.005)} & \makecell{0.415 \\[-3pt] \scriptsize (0.054)} & \makecell{0.797 \\[-3pt] \scriptsize (0.006)} & \makecell{0.495 \\[-3pt] \scriptsize (0.022)} & \makecell{0.582 \\[-3pt] \scriptsize (0.012)} & \makecell{0.142 \\[-3pt] \scriptsize (0.057)}\\
 & Qwen3-8B & \makecell{\textbf{0.884} \\[-3pt] \scriptsize (0.002)} & \makecell{\textbf{0.579} \\[-3pt] \scriptsize (0.049)} & \makecell{0.878 \\[-3pt] \scriptsize (0.002)} & \makecell{0.570 \\[-3pt] \scriptsize (0.053)} & \makecell{0.878 \\[-3pt] \scriptsize (0.003)} & \makecell{0.506 \\[-3pt] \scriptsize (0.068)} & \makecell{0.814 \\[-3pt] \scriptsize (0.002)} & \makecell{0.433 \\[-3pt] \scriptsize (0.042)} & \makecell{0.669 \\[-3pt] \scriptsize (0.014)} & \makecell{0.217 \\[-3pt] \scriptsize (0.046)}\\
 & Llama-3.2-3B-Instruct & \makecell{\textbf{0.808} \\[-3pt] \scriptsize (0.005)} & \makecell{\textbf{0.470} \\[-3pt] \scriptsize (0.062)} & \makecell{0.797 \\[-3pt] \scriptsize (0.007)} & \makecell{0.455 \\[-3pt] \scriptsize (0.067)} & \makecell{0.781 \\[-3pt] \scriptsize (0.010)} & \makecell{0.235 \\[-3pt] \scriptsize (0.057)} & \makecell{0.683 \\[-3pt] \scriptsize (0.007)} & \makecell{0.243 \\[-3pt] \scriptsize (0.055)} & \makecell{0.540 \\[-3pt] \scriptsize (0.006)} & \makecell{0.063 \\[-3pt] \scriptsize (0.019)}\\
\bottomrule
\end{tabular}
}
\caption{Performance on Empirical datasets. Best AUC (higher) and Spearman $\rho$ (higher) per model are bolded.}
\label{tab:empirical}
\end{table*}

\subsection{Notes and Recommendations}
\paragraph{Prompt choice matters.}
We use a single English prompt for both languages to keep training/evaluation invariant across cases; the Chinese template can be swapped in with no other changes. 
Pair-level prompting (single sequence for $(c,r)$) avoids the need for cross-encoder similarity and empirically stabilizes the scorer.

\paragraph{KDE caveat.}
Even with StandardScaler and PCA (128-d) plus an automatic bandwidth heuristic, KDE under-performs in high-dimensional spaces; results corroborate this limitation across both languages.

\paragraph{Replicability.}
All methods share \emph{the same} negatives, scorer head, optimizer, and early-stopping rule; random seed is fixed ($42$). 
We rely on vLLM for embeddings (\texttt{max\_model\_len=2048}, GPU util $0.90$) and compute AUC/Spearman with standard \texttt{sklearn} APIs.

\end{document}